\documentclass[11pt]{article} 

\usepackage{amsmath,amsfonts,bm}

\def\eqref#1{equation~\ref{#1}}

\def\1{\bm{1}}

\DeclareMathAlphabet{\mathsfit}{\encodingdefault}{\sfdefault}{m}{sl}
\SetMathAlphabet{\mathsfit}{bold}{\encodingdefault}{\sfdefault}{bx}{n}

\newcommand{\E}{\mathbb{E}}

\newcommand{\R}{\mathbb{R}}

\usepackage{fullpage}
\usepackage{amsthm}
\usepackage{amsmath, amssymb, graphicx, url}
\usepackage[square,sort,comma,numbers]{natbib}
\usepackage{color}
\usepackage{caption}
\usepackage{booktabs}
\usepackage{multirow}
\usepackage{enumitem}

\usepackage{mathtools}

\usepackage{authblk}

\usepackage{subfig}

\usepackage{./MScMac}

\usepackage{hyperref}
\definecolor{refcol}{rgb}{0.1,0,0.6}
\hypersetup{colorlinks}
\hypersetup{linkcolor=refcol, citecolor=refcol}

\usepackage[capitalize,noabbrev]{cleveref}

\theoremstyle{plain}
\newtheorem{theorem}{Theorem}[section]
\newtheorem{proposition}[theorem]{Proposition}
\newtheorem{lemma}[theorem]{Lemma}
\newtheorem{corollary}[theorem]{Corollary}
\theoremstyle{definition}
\newtheorem{definition}[theorem]{Definition}
\newtheorem{assumption}[theorem]{Assumption}
\theoremstyle{remark}
\newtheorem{remark}[theorem]{Remark}
\newtheorem{example}[theorem]{Example}
\allowdisplaybreaks

\title{Almost sure convergence rates of stochastic gradient methods under gradient domination}

\author[$1$,$\ast$]{Simon Weissmann}
\author[$1$,$\ast$]{Sara Klein}
\author[$2$]{Waïss Azizian}
\author[$1$]{Leif Döring}
\date{\ }
\affil[$1$]{\normalsize
Institute of Mathematics, University of Mannheim, 
68138 Mannheim, Germany\\
\texttt{\{simon.weissmann, sara.klein, leif.doering\}@uni-mannheim.de}
\vspace{0.25cm}
}

\affil[$2$]{\normalsize
CNRS, Grenoble INP, Université Grenoble Alpes, 38000 Grenoble, France\\
 \texttt{waiss.azizian@univ-grenoble-alpes.fr} \vspace{0.25cm}
}
\affil[$\ast$]{\normalsize
These authors contributed equally to this work
}

\begin{document}

\maketitle

\begin{abstract}
Stochastic gradient methods are among the most important algorithms in training machine learning problems. While classical assumptions such as strong convexity allow a simple analysis they are rarely satisfied in applications. In recent years, global and local gradient domination properties have shown to be a more realistic replacement of strong convexity. They were proved to hold in diverse settings such as (simple) policy gradient methods in reinforcement learning and training of deep neural networks with analytic activation functions. We prove almost sure convergence rates $f(X_n)-f^*\in o\big( n^{-\frac{1}{4\beta-1}+\epsilon}\big)$ of the last iterate for stochastic gradient descent (with and without momentum) under global and local $\beta$-gradient domination assumptions. The almost sure rates get arbitrarily close to recent rates in expectation. Finally, we demonstrate how to apply our results to the training task in both supervised and reinforcement learning.
\end{abstract}

\section{Introduction}
\label{sec:intro}
First-order methods to minimize an objective function $f$ have played a central role in the success of machine learning. This is accompanied by a growing interest in convergence statements particularly for stochastic gradient methods in different settings. To ensure convergence to the global optimum some kind of convexity assumption on the objective function is required. Especially in machine learning problems the standard (strong) convexity assumption is nearly never fulfilled. However, it is well known that achieving convergence towards global optima is still possible under a weaker assumption, namely under the gradient domination property, often referred to as Polyak-\L ojasiewicz  (PL)-inequality \citep{POLYAK1963}. 
Also in reinforcement learning, multiple results have shown that the objective function for policy gradient methods, under specific parametrizations, fulfills a weak type of gradient domination and therefore provably achieve convergence towards the global optimum \citep{mei2020b, mei2021non-uniformPL, fatkhullin2023, klein2023stationarity}.  
Improving the understanding of rates and optimal step size choices for stochastic first order methods is of significant interest for the machine learning and reinforcement learning community. Many classical results identify convergence rates 
for the expected error 
$\E[f(X_n)-f^*]$. 
In the present article, we focus on almost sure convergence rates for the error $f(X_n)-f^*$ in stochastic gradient schemes under weak gradient domination. The contribution of this work is as follows:
\begin{enumerate} 
\item[(i)] Under global gradient domination with parameter $\beta$ (on the entire function domain), we prove that the last iterate of stochastic gradient descent (SGD) and stochastic heavy ball (SHB) converge almost surely and in expectation towards the global optimum with rate arbitrarily close to $o(n^{-\frac{1}{4\beta-1}})$. The almost sure and expectation rates of convergence that we obtain depend on the gradient domination parameter $\beta$ and are the same for both algorithms and convergence types. For SGD this rate is arbitrarily close to the tight upper bound known in expectation \citep{fatkhullin2022}, while the almost sure convergence rate is new for the (weak) gradient domination assumption (see \cref{thm:global_convergencerate_SGD_beta_allgemein} and discussion afterwards). To the best of our knowledge for SHB this is the first convergence result towards global optima under (weak) gradient domination, for both almost sure convergence and convergence in expectation (see \cref{thm:SHB-global}).

\item[(ii)] We consider the case where the gradient domination property holds only locally, either around stationary points or around global minima. We provide the first local convergence rates under these settings: we prove that SGD remains within the good local region with high probability and, conditioned on this event, we obtain converges rates almost surely and in expectation (see \cref{thm:conv-SGD-local-PL-xast} and \cref{thm:conv-SGD-local-PL-fast}). 

\item[(iii)] Our local setting covers generic classes of functions. In particular, we demonstrate that it encompasses the training task of deep neural networks with analytic activation functions in supervised learning. Our result illustrates that the iterates of SGD are likely to become trapped in areas of local minima when the step size is small. We verify under mild conditions, that SGD converges to local minima with given convergence speed (see Corollary~\ref{cor:DNN}).

\item[(iv)] Finally, we derive novel convergence results for policy gradient methods in reinforcement learning. We show that local gradient domination holds around the global optimum for the softmax parametrization resulting in the first local convergence rate for stochastic policy gradient with arbitrary batch-size (see Corollary~\ref{cor:discounted-convergence}).

\end{enumerate}

We summarize the contributions of this paper in~\cref{tab:my-table}. These findings are also illustrated in a numerical toy experiment in \cref{app:numerics}, where we have implemented SGD and SHB for monomials with increasing degree. 
\begin{table*}[htb!]
\caption{Summary of known and new results. Table presents convergence rates for tuned step size ($\epsilon>0$ arbitrarily small). Dom.: gradient domination holds locally or globally; local*: additional assumption on $\gamma_1$ required and results holds only locally. a.s.: almost surely; $\mathbb E$: in expectation. Ref.: for some cited results minor adjustments are necessary.}
\label{tab:my-table}
\vskip 0.15in
\begin{center}
\begin{tabular}{@{}c @{}c @{}c @{} l  @{} c c@{} l@{}}
\toprule
$\mathbf{\beta}    $                    & \textbf{Step size}                                                                                                                 & \textbf{Rate}                                                                  & \textbf{Dom. }                                 & \textbf{Algo.}                 & \textbf{Conv. }       & \textbf{Ref.}       \\ \midrule
\multirow{12}{*}{$\frac 12$}    & \multirow{12}{*}{\begin{tabular}[c]{@{}l@{}}$\Theta\left( n^{-1+\epsilon}\right)$\end{tabular}}                    & \multirow{12}{*}{$o\left( n^{-1+\epsilon}\right)$}                   & \multirow{8}{*}{global}     & \multirow{4}{*}{SGD} & \multirow{2}{*}{a.s.}        & \Cref{thm:global_convergencerate_SGD_beta_allgemein} (i); \\ 
& & & & & & \citet[Thm. 1]{LY2022}    \\ \cline{6-7}
                               &                                                                                                                           &                                                                       &                         &                                            & \multirow{2}{*}{$\mathbb E$}  & \Cref{thm:global_convergencerate_SGD_beta_allgemein} (ii); \\
                               
& & & & & & \citet[Thm. 3]{khaled2022better}  \\\cline{5-7}
                               &                                                                                                                           &                                                                       &                         &                        \multirow{4}{*}{SHB} & \multirow{2}{*}{a.s.}          &  \Cref{thm:SHB-global} (i); \\
                             
& & & & & & \citet[Thm. 2]{LY2022}    \\\cline{6-7}
                               &                                                                                                                           &                                                                       &                         &                        &                      \multirow{2}{*}{$\mathbb E$} & \Cref{thm:SHB-global} (ii);\\ 
& & & & & & \citet[Thm. 4.3]{LIANG2023830}   \\ \cline{4-7}
                               &                                                                                                                           &                                                                       & \multirow{4}{*}{local*}   & \multirow{4}{*}{SGD}                 & \multirow{2}{*}{a.s.}         & \Cref{thm:conv-SGD-local-PL-xast} (ii); \\  & & & & & &  \Cref{thm:conv-SGD-local-PL-fast} (ii) \\\cline{6-7}
                               &                                                                                                                           &                                                                       &                         &                        &                    \multirow{2}{*}{$\mathbb E$} & \Cref{thm:conv-SGD-local-PL-xast} (iii); \Cref{thm:conv-SGD-local-PL-fast} (iii);\\
                               
& & & & & & \citet[Thm. 4]{mertikopoulos2020sure} \\
\hline
\multirow{12}{*}{$(\frac12,1]$} & \multirow{12}{*}{\begin{tabular}[c]{@{}l@{}}$\Theta \left( n^{-\frac{2\beta}{4\beta-1}}\right)$\end{tabular}} & \multirow{12}{*}{$ o\left( n^{-\frac{1}{4\beta-1}+\epsilon}\right)$} & \multirow{6}{*}{global}     & \multirow{4}{*}{SGD} & \multirow{2}{*}{a.s.}       & \multirow{2}{*}{\Cref{thm:global_convergencerate_SGD_beta_allgemein} (i)}    \\ \\\cline{6-7}
                               &                                                                                                                           &                                                                       &                         &                        &                      \multirow{2}{*}{$\mathbb E$} & \Cref{thm:global_convergencerate_SGD_beta_allgemein} (ii); \\
                               
& & & & & & \citet[Cor. 1]{fatkhullin2022} \\\cline{5-7}
                               &                                                                                                                           &                                                                       &                         &                        \multirow{4}{*}{SHB} & \multirow{2}{*}{a.s.}         & \multirow{2}{*}{\Cref{thm:SHB-global} (i)}    \\ \\ \cline{6-7}
                               &                                                                                                                           &                                                                       &                         &                        &           \multirow{2}{*}{$\mathbb E$} & \multirow{2}{*}{\Cref{thm:SHB-global} (ii)}   \\ \\ \cline{4-7}
                               &                                                                                                                           &                                                                       & \multirow{4}{*}{local*}  & \multirow{4}{*}{SGD} & \multirow{2}{*}{a.s.}        & \Cref{thm:conv-SGD-local-PL-xast} (ii);\\  & & & & & & \Cref{thm:conv-SGD-local-PL-fast} (ii)   \\ \cline{6-7}
                               &                                                                                                                           &                                                                       &                         &                        &           \multirow{2}{*}{$\mathbb E$} & \Cref{thm:conv-SGD-local-PL-xast} (iii); \\  & & & & & &  \Cref{thm:conv-SGD-local-PL-fast} (iii) \\ \bottomrule
\end{tabular}
\end{center}
\vskip -0.1in
\end{table*}

\subsection{Literature Review and Classification of our Contribution}\label{subsec:literature}
The roots of stochastic gradient methods trace back to \citet{monro51}. Since then, various variants of SGD have been established as fundamental algorithms for optimizing complex models in the realm of machine learning. We refer to \citet{bottou18} and \citet{garrigos2024handbookconvergencetheoremsstochastic} for a detailed overview. 

We start the review with the literature deriving convergence rates in expectation for SGD. Under the assumptions of smoothness and (strong) convexity \citet{poliak1987introduction,moulines11,Nguyen2018,wang2021convergence,liu23} studied convergence rates towards global optima. Moreover, many articles additionally analyze the non-convex case and prove convergence rates for the gradient norm towards zero \citep{ghadimi-lan13,li2023,liu23,nguyen2023}. 

Notably, several other results regarding convergence of SGD towards global optima have been established under the gradient domination setting \citep{10.1007/978-3-319-46128-1_50}. \citet{Bassily18,Liu2020LossLA} demonstrate exponential convergence rates in expectation in the overparameterized setting under strong gradient domination.
See also \citet{madden2021highprobability}, where high-probability convergence bounds are shown
using the (strong) gradient domination property. 
Further high-probability bounds on the approximation error are provided in \citet{pmlr-v162-scaman22a} under a generalized gradient domination property, the so-called Separable-\L ojasiewicz assumption, fulfilled by smooth neural networks. In \citet{Lei2020} also strong gradient domination is assumed, where the smoothness assumption is weakened through $\alpha$-Hölder continuity, achieving a rate of $O(\frac{1}{n^\alpha})$ in expectation. The \cref{eq:ABC-condition} condition is introduced in \citet{khaled2022better} and where $O(\frac1n)$ convergence is shown under strong gradient domination. Furthermore, \citet{fatkhullin2022} and \citet{pmlr-v134-fontaine21a} consider generalizations of gradient domination that include our definition as a special case. They derive convergence rates in expectation which we encompass with our result and extend to almost sure convergence (see also the discussion behind \cref{thm:global_convergencerate_SGD_beta_allgemein}). Finally, \citet{masiha2024complexityminimizingprojectedgradientdominatedfunctions} established both upper and lower bounds for projected gradient descent under gradient domination.

All the results mentioned so far consider convergence in expectation or high-probability bounds, although originally, motivated by \citet{RS1971}, research commenced with the quest for almost sure convergence rates for gradient methods.
In recent years, \citet{pmlr-v134-sebbouh21a} and, building upon it, \citet{LY2022} derive almost sure convergence rates towards global optima under convexity for SGD and SHB. In addition, \citet{LY2022} study SHB and Nesterov acceleration under PL.
Returning the attention back to SGD with respect to gradient domination also some almost sure convergence results have been established. 
As an extension to the PL-type gradient domination, in \citet{chouzenoux2023kurdykalojasiewicz} the so-called KL property is assumed, which contains gradient domination as a special case. The authors demonstrate almost sure convergence to a critical point, though without a rate.  
To conclude, to the best of our knowledge the derived almost sure convergence rate under gradient domination in \cref{thm:global_convergencerate_SGD_beta_allgemein} is novel.

Next, we want to provide further insights to the literature regarding SHB. In the realm of momentum methods, Polyak's Heavy-Ball Method (HBM) \citep{polyak1964HB} and Nesterov's accelerated gradient method \citep{nesterov1983method} stand out as a foundational contribution. The authors of \citet{gadat16} provide a detailed description of the stochastic formulation of HBM and establish almost sure convergence but without giving a rate. In \citet{yang2016unified,orvieto2020role,Yan18,mai2020convergence,zhou20} convergence rates in expectation are shown in (strongly) convex and non-convex settings, where the non-convex analysis covers convergence of the norm of the gradient. 
  Convergence of momentum methods under the strong gradient domination property and linear convergence due to an overparametrized machine learning setting is shown in \citet{gess2023convergence}. \citet{LIANG2023830} determine $O(\frac{1}{n})$ convergence rate for SHB under strong gradient domination. Our main result for SHB presented in Theorem~\ref{thm:SHB-global} describes almost sure convergence and convergence in expectation under global gradient domination. Both result are quantified with a given rate of convergence. 
 
In the following, we aim to differentiate 
our local convergence analysis from existing results in the literature. 
There has been a lot of effort to derive local convergence guarantees for stochastic first order methods. 
In \citet{dereich2021convergence} almost sure convergence of SGD to a stationary point under the local gradient domination (for $x^\ast$) is demonstrated, provided that the process ($X_n$) remains local, albeit without a rate. A local analysis of SGD towards minima without any gradient domination assumption is presented in \citet{fehrman2020}. Instead, a rank assumption is imposed on the Hessian, and mini-batches, along with resampling, are leveraged to ensure convergence to the global optimum with high probability. The resulting rate does not converge to zero and requires an increasing batch size. Finally, under the global Lipschitz assumption on the objective, it can be verified that SGD almost surely converges to a stationary point as demonstrated by \citet{mertikopoulos2020sure}. In the same work the authors derive a local convergence analysis under local strong convexity. Our analysis in Section~\ref{sec:localSGD} builds upon \citet{mertikopoulos2020sure} and generalizes their results to the local gradient domination property. In our analysis we distinguish the cases where the local gradient domination property holds in a neighbourhood of a local minimum or in the neighbourhood of the global optimum respectively. Finally, we would like to acknowledge that related results have been independently obtained in the recent preprint \cite{qiu2024convergencesgdmomentumnonconvex}.

For the application in the training of DNNs, it is worth noting that local convergence of SGD has been analyzed under stronger variants of gradient domination by \cite{Wojtowytsch2023,an2024convergence}. Due to the stronger form of gradient domination, specific sub-classes of DNNs need to be considered to verify these assumptions whereas our result is only constrained to analytic activation functions. Under the machine learning noise conditions in \cite{Wojtowytsch2023}, convergence toward zero loss with high probability is shown, provided that the initial loss is sufficiently small. In contrast, \cite{an2024convergence} demonstrate convergence towards zero loss under initialization in a local (strong) Łojasiewicz region. Indeed, one can construct DNNs satisfying the latter condition \citep{chatterjee2022convergence}. 

For the application in reinforcement learning, recent results showed that choosing the tabular softmax parametrization in policy gradient (PG) algorithms results in objective functions which fulfill a non-uniform gradient domination property \citep{mei2020b,yuan2022general, klein2023stationarity}. While convergence of PG for exact gradients is well understood, convergence rates for stochastic PG are rare and mostly require very large batch sizes \citep{ding2022beyondEG,klein2023stationarity,ding2023local}. It is noteworthy that a similar local analysis for stochastic policy gradient under entropy regularization is presented in \citet{ding2023local}. Their local result is also based on \citet{mertikopoulos2020sure}, but requires an increasing batch size sequence to obtain $O(\frac{1}{n})$-convergence towards the regularized optimum with high probability.
In contrast, we 
consider both the unregularized and entropy regularized setting and observe that one can also achieve convergence arbitrarily close to $o(\frac{1}{n})$ without the need for an increasing batch size. Moreover, the local convergence occurs almost surely on an event with high probability.

\section{Mathematical Background - Optimization under Gradient Domination}\label{sec:mathbackground}

We consider the problem of solving the minimization problem of the form
\begin{equation}\label{eq:objective}
    \min_{x\in\R^d} f(x)\, ,   
\end{equation}
where $f:\R^d \to \R$ denotes the objective function of interest. Throughout this paper we assume that the objective function is bounded from below by $f^\ast= \inf_{x\in\R^d} f(x) >-\infty$ and satisfies the classical L-smoothness assumption (either locally or globally):
\begin{assumption}\label{ass:L-smooth}
 The objective function $f:\R^d\to\R$ is differentiable and the gradient $\nabla f$ is 
 \begin{enumerate}[label=(\roman*)]
 \item\label{ass:L-smooth-global} globally $L$-Lipschitz continuous, i.e.~there exists $L>0$ such that $\|\nabla f(x)-\nabla f(y)\|\le L \|x-y\|$ for all $x,y\in\R^d$.
 \item\label{ass:L-smooth-local} locally $L$-Lipschitz continuous, i.e.~for all $R>0$ there exists $L(R)>0$ such that $\|\nabla f(x)-\nabla f(y)\|\le L(R) \|x-y\|$ for all $x,y\in\R^d$ with $|x|,|y|\le R$.
 \end{enumerate}
\end{assumption}
Using (global) $L$-smoothness, the descent lemma provides the inequality
\begin{equation}\label{eq:descent-lemma}
    f(y) \leq f(x) + \langle \nabla f(x), y-x\rangle + \frac{L}{2} \lVert y-x\rVert^2
\end{equation}
which is a fundamental instrument to analyze first order optimization methods. As a motivation recall the iterative update generated by gradient descent with constant step size $\gamma\le \frac1L$, i.e.
\[x_{n+1} = x_n - \gamma \nabla f(x_n),\quad x_1\in\R^d\,. \]
Applying \cref{eq:descent-lemma} and the iteration scheme yields the iterative descent property 
\[\left[f(x_{n+1}) - f^\ast\right] \le\left[f(x_n) - f^\ast\right] - \frac{\gamma}2 \|\nabla f(x_n)\|^2\,. \]
Under further strong convexity assumption it is classical to show that the gradient descent algorithm converges to a global minimum at a linear rate. 
In order to derive a convergence rate without assuming convexity of $f$ one can use dominating relations of the gradient $\nabla f(x)$ with respect to the optimality gap $f(x)-f^\ast$. In particular, as demonstrated in \cite{10.1007/978-3-319-46128-1_50} it is nowadays well-known that gradient descent converges linearly under the PL-condition which assumes that there exists $c>0$ such that for all $x\in\R^d$ there holds
\begin{equation}\label{eq:strong_PL}
\|\nabla f(x)\| \ge c(f(x)-f^\ast)^\beta 
\end{equation}
with exponent $\beta = 1/2$ \citep{POLYAK1963}. It is worthwhile to emphasize that \cref{eq:strong_PL} is weaker than strong convexity, the classical textbook assumption that fails for many applications. Under the PL-condition the iterative descent property can be written as a recursion:
\[ \left[f(x_{n+1}) - f^\ast\right] \le \left(1-\frac{\gamma c}{2}\right)\left[f(x_n) - f^\ast\right]\,.\] 
In fact, there are many works analyzing (stochastic) first order methods under the weaker \L ojasiewicz condition formulated in (local) areas around stationary points $x^\ast$ and exponents $\beta\in[1/2,1]$ \cite{pmlr-v49-lee16,fatkhullin2022,pmlr-v162-scaman22a,Wilson2019AcceleratingRG}. 
For general $\beta\in[1/2,1]$ the recursive descent property reads as 
\[\left[f(x_{n+1}) - f^\ast\right] \le \left[f(x_n) - f^\ast\right] - \frac{\gamma c}{2}\left[f(x_n) - f^\ast\right]^{2\beta} \]
leading to sub-linear convergence for $\beta>1/2$. 

For the purpose of our analysis of stochastic gradient methods, we collect the following types of global and local gradient domination properties.

\begin{definition}\label{def:PL-defs}
    Let $f:\R^d\to\R$ be continuously differentiable with $f^\ast = \inf_{x\in\R^d}\ f(x) >-\infty$. 
    \begin{enumerate}
        \item We say that $f$ satisfies the global gradient domination property with parameter $\beta\in[\frac12,1]$ if there exists $c>0$ such that for all $x\in\R^d$ it holds true that
        \[\|\nabla f(x)\| \ge c (f(x)-f^\ast)^\beta. \]
        \item Let $x^\ast\in\R^d$ be a stationary point, i.e.~$\nabla f(x^\ast)=0$. We say that $f$ satisfies a local gradient domination property in $x^\ast$ with parameter $\beta_{x^\ast}\in[\frac12,1]$ if 
        a radius $r_{x^\ast}>0$ and a constant $c_{x^\ast}>0$ such that 
        \begin{equation}\label{eq:loj}
        \|\nabla f(x)\| \ge c_{x^\ast} |f(x)-f(x^\ast)|^{\beta_{x^\ast} }
        \end{equation}
        for all $x\in \cB_{r_{x^\ast}}(x^\ast) = \{y\in\R^d : \lVert x^\ast -y\rVert \leq r_{x^\ast}\}.$
        We say that $f$ satisfies a local gradient domination property in $f^\ast$ with parameter $\beta\in[\frac12,1]$ if there exist a radius $r>0$ and a constant $c>0$ such that  
         \[\|\nabla f(x)\| \ge c (f(x)-f^\ast)^\beta \]
         for all $x \in\cB_r^\ast = \{y\in\R^d: f(y)-f^\ast \leq r\}$.
    \end{enumerate}
\end{definition}
\begin{remark}\label{rem:PL-defs}
    If $\beta = \frac{1}{2}$, we will call the gradient domination \textit{strong} since it is implied by strong convexity. In contrast, we call the gradient domination \textit{weak} for $\beta \in(\frac{1}{2},1]$.
    Moreover, note that for the local gradient domination property in $x^\ast$ the parameters $r$ and $c$ may depend on $x^\ast$. Furthermore, we emphasize that for the definition of the local gradient domination in $f^\ast$ we do not require the existence of $x^\ast \in \arg\min_{x\in\R^d} f(x)$. 
\end{remark}

The PL-condition mentioned above is a special case of the general global gradient domination property for $\beta = \frac 12$. In \citet{lojasiewicz1965} 
it has been demonstrated that all analytic functions satisfy the local gradient domination property, emphasizing the particular significance of the local case. Further, it has been proved that all overparametrized neural networks fulfill the local gradient domination property \citep{Liu2020LossLA}. See also \citet{madden2021highprobability,dereich2021convergence,NEURIPS2021_42299f06} and references therein for the application of (strong) gradient domination to (deep) neural networks. 
In \citet{fatkhullin2022,ABRS2010,BST2014,NEURIPS2018_b4568df2} examples of functions are discussed that fulfill the (weak) gradient domination property. For instance, one-dimensional monomials $f(x)=|x|^p$, $p\geq 2$, satisfy the weak global gradient domination property with $\beta = \frac{p-1}{p}$.
We refer to \citet[Appendix A]{fatkhullin2022} for a longer list of globally gradient dominated functions including convex and non-convex functions. 
Notably, in reinforcement learning it is known that the tabular softmax parametrization leads to a parametrized value function that satisfies the so-called "non-uniform" PL-inequality \citep{mei2020b, mei2021non-uniformPL, klein2023stationarity}. In Section~\ref{sec:application-RL} we will show how this non-uniform gradient domination implies local gradient domination in $f^*$. 
This renders our local analysis of stochastic gradient methods specifically applicable in RL. As mentioned earlier, since every analytic function already satisfies local gradient domination, we expect that the local analysis can encompass further parametrizations, such as neural networks.

\subsection{Assumptions on the Stochastic First Order Oracle}
 Let $(\Omega, \cF, \bbP)$ be an underlying probability space. In general, we assume that we can access the exact gradient $\nabla f(x)$ through a stochastic first order oracle $V:\R^d\times M\to\R^d$ defined by
\begin{equation} \label{eq:sfoo}
V(x,m) = \nabla f(x) + Z(x,m), \quad x\in\R^d,\ m\in M\, , 
\end{equation}
where $(M,\mathcal M)$ is a measurable  space,  $Z:\R^d\times M\to\R^d$ is a state dependent $\cB(\R^d)\otimes \mathcal M / \cB(\R^d)$-measurable mapping describing the error to the exact gradient $\nabla f$. The stochastic gradient evaluation is then modelled through $V(x,\zeta)$, where the random variable $\zeta:\Omega\to M$ is independent of the state $x\in\R^d$. We make the following unbiasedness and second moment assumption:

\begin{assumption}\label{ass:ABC}
We assume that for each $x\in\R^d$ it holds that
\[\E[Z(x,\zeta)] := \int_\Omega Z(x,\zeta(\omega)) {\mathrm{d}}\mathbb P(\omega) = 0 \]
and there exist non-negative constants $A,B$ and $C$ such that for all $x\in\R^d$ it holds that
    \begin{equation}\label{eq:ABC-condition}
        \E[\lVert V(x,\zeta) \rVert^2] \leq A(f(x)-f^\ast) + B \lVert \nabla f(x) \rVert^2 + C\, . \tag{ABC}
    \end{equation}
\end{assumption}

It is worth noting that the \cref{eq:ABC-condition} assumption is a generalization of the bounded variance assumption that appears for $A=B=0$. It was introduced by \citet{khaled2022better} as expected smoothness condition and shown to be the weakest assumption among many others.
\subsection{Stochastic Gradient Methods}
The following two classical optimization algorithms will be analyzed in this article. Both algorithms are described as discrete time stochastic process $(X_n)_{n\in\N}$ driven by noisy gradient evaluations in \cref{eq:sfoo}. In each iteration, we assume that the stochastic first order oracle is accessed through the evaluation of $\zeta_{n+1}$ which is a copy of $\zeta$ independent from the current state $X_n$. 

The stochastic gradient descent (SGD) scheme is given by the stochastic update 
\begin{equation*}
    X_{n+1} = X_n - \gamma_n\, V(X_n,\zeta_{n+1})\,, 
\end{equation*}
where $X_1$ is a $\R^d$-valued random vector which denotes the initial state. To keep the notation simple, we will 
introduce $V_{n+1}(X_n):= V(X_n,\zeta_{n+1})$ suppressing the explicit noise representation through $(\zeta_n)_{n\in\N}$ in the following. The iterative update formula then reads as
\begin{equation}\label{eq:SGD}
    X_{n+1} = X_n - \gamma_n\, V_{n+1}(X_n). \tag{SGD}
\end{equation}

The iterative scheme of stochastic heavy ball (SHB) is defined by
\begin{equation}\label{eq:SHB}
    X_{n+1} = X_n - \gamma_n V_{n+1}(X_n) + \nu(X_n-X_{n-1})\, ,\tag{SHB}
\end{equation}
with initial $\R^d$-valued random vector $X_1$. The additional summand is called the momentum term with momentum parameter $\nu\in[0,1)$. In both cases, $(\gamma_n)_{n\in\N}$ denotes a sequence of positive step sizes and 
we denote by $(\mathcal F_n)_{n\in\N}$ the natural filtration induced by the process $(X_n)_{n\in\N}$. Note that we adopt the convention where the set of natural numbers $\N$ refers to the positive integers excluding zero. When necessary, we explicitly define $\N_0:=\N\cup \{0\}$. 

\begin{example}[Expected risk minimization]\label{ex:example}
In order to give more insights into the considered setting of our stochastic first order oracle we formulate a specific one based on expected risk minimization. In expected risk minimization we are interested in minimizing an objective function of the form
 \[f(x) = \E[F(x,\zeta)] = \int_{\Omega} F(x,\zeta(\omega)) \, d\mathbb P(\omega)\, \]
 where $F:\R^d\times M\to \R$ is $\cB(\R^d)\otimes \mathcal M / \cB(\R)$-measurable. In our notation the stochastic first order oracle then takes the form
 \[V(x,\zeta) =  \nabla f(x) + \left(\nabla_x F(x,\zeta)-\nabla f(x)\right) = \nabla_x F(x,\zeta) \]
 and the iterative update of SGD reads as
 \[ X_{n+1} = X_n - \gamma_n \nabla_x F(X_n,\zeta_{n+1}) \]
 with a sequence of independent and identically distributed $(\zeta_n)$. The iterative scheme of SHB can be written in similar way. Note that this scenario also includes empirical risk minimization where the objective function takes a finite sum form, with $\zeta \sim \mathcal U(\{1,\dots,N\})$,
 \[f(x) = \frac{1}{N}\sum_{i=1}^N F(x,i) = \E[ F(x,\zeta)]\,. \] 
\end{example}

Exemplifying the SGD method, we now illustrate the typical steps of the convergence analysis for first-order optimization methods. First, the smoothness of the function $f$ is exploited by applying the descent inequality \cref{eq:descent-lemma} to the iteration scheme and then applying conditional expectations,
\begin{equation*}
    \E[f(X_{n+1}) \mid \mathcal F_n] 
    \leq f(X_n) - \gamma_n \lVert \nabla f(X_n)\rVert^2 
    + \frac{L \gamma_n^2}{2} \E[ \lVert V_{n+1}(X_n)\rVert^2 \mid \mathcal F_n].
\end{equation*}
Next, $f^\ast$ is subtracted on both sides and the variance term of the stochastic gradient is controlled through the \cref{eq:ABC-condition} condition:
\begin{equation}\label{eq:smooth+ABC} 
        \E[f(X_{n+1})-f^\ast \mid \mathcal F_n] \leq \Big(1+\frac{LA\gamma_n^2}{2}\Big) (f(X_n)-f^\ast) - \Big(\gamma_n-\frac{BL\gamma_n^2}{2}\Big) \lVert \nabla f(X_n)\rVert^2 + \frac{LC\gamma_n^2}{2}.
\end{equation} 

Without further assumptions this inequality can now be used to show that the gradient $\nabla f(X_n)$ converges almost surely to zero. In order to obtain convergence towards a global optimum additional assumptions are needed. For instance, it is sufficient to incorporate the global gradient domination property defined in Definition~\ref{def:PL-defs} which yields an iterative inequality of the form
\begin{equation}\label{eq:smooth+ABC+gradient-dom} 
    \E[f(X_{n+1})-f^\ast \mid \mathcal F_n]\leq \Big(1+\frac{LA\gamma_n^2}{2}\Big) (f(X_n)-f^\ast) 
    - \Big(\gamma_n-\frac{BL\gamma_n^2}{2}\Big) c^2 (f(X_n)-f^\ast)^{2\beta} + \frac{LC\gamma_n^2}{2}.
\end{equation} 
Typically, the expectation is taken on both sides of the inequality to derive a convergence rate in expectation by working with recursive inequalities. In this article we push the argument further. We combine smoothness and gradient domination with a variant of the Robbins-Siegmund Theorem (see Lemma~\ref{cor:RS_corollary}) to derive almost sure convergence rates.

\section{Preliminary Discussion on Super-Martingale Convergence Rates}\label{sec:prelim-on-supermartingales}

In the previous section, we have sketched how to combine the global gradient domination property with smoothness to derive a recursive inequality of the form
\[ \E[Y_{n+1}\mid \mathcal F_n] \le (1+c_1 \gg_n)Y_n - c_2\gamma_n Y_n^{2\beta} + c_3\gamma_n^2\,,\]
where $Y_n:= f(X_n)-f^\ast$. For analysing these inequalities, we must deal separately with the strong gradient domination case ($\beta= \frac 12$) and the weak gradient domination case ($\beta > \frac12$) to avoid divisions by zero. For the former case the recursive inequality simplifies, whereas a more complex analysis is required for the latter. 
To establish almost sure convergence rates we employ convergence lemmas for super-martingales based on the Robbins-Sigmund Theorem. This methodology has been introduced in \citet{pmlr-v134-sebbouh21a} and further utilized in \citet{LY2022} to analyze SGD and SHB under (strong) convexity. It is noteworthy to mention that their analysis is under the similar noise assumption formulated in \Cref{ass:ABC}.

In the following, we illustrate how to extend the arguments to convergence under the global gradient domination property. 
Here is our super-martingale result that also encompasses \citet[Lemma 1]{LY2022} when $\beta=\frac{1}{2}$ for completeness: 
\begin{lemma}\label{lem:as_rate_extension_beta_allgemein}
    Let $(Y_n)_{n\in\N}$ be a sequence of non-negative random variables on an underlying probability space $(\Omega,\cF,\bbP)$ with natural filtration $(\mathcal F_n)_{n\in\N}$ and suppose there exists $\beta \in [\frac12,1]$, $c_1,c_3\geq0$ and $c_2>0$ such that
    \[ \E[Y_{n+1}\mid \mathcal F_n] \le (1+c_1\gamma_n^2)Y_n-c_2\gamma_n Y_n^{2\beta} + c_3\gamma_n^2\,,\] 
    for all $n\ge1$, where $\gamma_n = \Theta(\frac{1}{n^{\theta}})$ for some fixed  $\theta \in \left(\frac12,1\right)$.
    Then, for any 
    \[\eta \in\begin{cases} 
    \left(\max\{2-2\gt, \frac{\gt+2\beta-2}{2\beta-1}\},1\right)&: \beta\in(\frac12,1]\\
    (2-2\gt,1) &:\beta=\frac12 
    \end{cases}\,,\]
    $(Y_n)_{n\in\N}$ vanishes almost surely with
    $Y_n \in o\left(\frac{1}{n^{1-\eta}}\right)$.
\end{lemma}
The proof of this lemma is based on the well-known Robbins-Siegmund theorem \cite[Theorem 1]{RS1971} and is provided in full detail in \Cref{app:proof-lem-as-rate-beta-allgemein}.

\section{Convergence for Global Gradient Domination Property}\label{sec:SGD_global}

In this part of the paper, 
we provide our global convergence result in a non-convex and globally smooth setting. 
Combining the recursive inequality \cref{eq:smooth+ABC+gradient-dom} with the super-martingale convergence result from Lemma~\ref{lem:as_rate_extension_beta_allgemein} leads to the following theorem. 
\begin{theorem}\label{thm:global_convergencerate_SGD_beta_allgemein}
    Suppose Assumption~\ref{ass:L-smooth} \ref{ass:L-smooth-global} and Assumption~\ref{ass:ABC} are fulfilled and let $f$ satisfy the global gradient domination property from Definition~\ref{def:PL-defs} with $\beta\in[\frac12,1]$. Denote by $(X_n)_{n\in\N}$ the sequence generated by \cref{eq:SGD} using a step size $\gamma_n = \Theta(\frac{1}{n^{\theta}})$ with $\theta\in \left( \frac{1}{2},1\right)$. For any 
    \[\eta \in\begin{cases} 
    \left(\max\{2-2\gt, \frac{\gt+2\beta-2}{2\beta-1}\},1\right)&: \beta\in(\frac12,1]\\
    (2-2\gt,1)&:\beta=\frac12 
    \end{cases}\] 
    it holds that
    \begin{equation*} 
    (i)\ \ f(X_n)-f^\ast \in o\left( \frac{1}{n^{1-\eta}}\right),\ \text{a.s.},\quad \text{and} \quad (ii)\ \ 
        \E[f(X_n) - f^\ast] \in o\left(\frac{1}{n^{1-\eta}}\right).
    \end{equation*}
\end{theorem}

\begin{proof}
Recall, in~\cref{sec:mathbackground} we derived~\cref{eq:smooth+ABC+gradient-dom}, 
\begin{align*}
    &\E[f(X_{n+1})-f^\ast \mid \mathcal F_n] \\
    &\leq \Big(1+\frac{LA\gamma_n^2}{2}\Big) (f(X_n)-f^\ast) - \Big(\gamma_n-\frac{BL\gamma_n^2}{2}\Big) c^2 (f(x)-f^\ast)^{2\beta} + \frac{LC\gamma_n^2}{2},
\end{align*}
which will be the basis of the proof of Theorem~\ref{thm:global_convergencerate_SGD_beta_allgemein}.
We treat again both cases for $\beta=\frac12$ and $\beta\in(\frac12,1]$ separately:

     \underline{$\beta = \frac12$:} In this case,~\cref{eq:smooth+ABC+gradient-dom} results in the super-martingale inequality 
      \[ \E[Y_{n+1} \mid \mathcal F_n] \leq \Big(1+\frac{LA\gamma_n^2}{2} - \gamma_n c^2 +\frac{BLc^2\gamma_n^2}{2}\Big) Y_n + \frac{LC\gamma_n^2}{2},\]
      with $Y_n = f(X_n)-f^\ast$. By the choice of $\gamma_n$ there exists $N>0$ and a constant $\tilde c >0$ such that $ \gamma_n c^2 - \frac{LA\gamma_n^2}{2} -  \frac{BLc^2\gamma_n^2}{2}\geq \tilde c \gamma_n$ for all $n\geq N$. Thus, 
      \[ \E[Y_{n+1} \mid \mathcal F_n] \leq \Big(1-\tilde c \gg_n\Big) Y_n + \frac{LC\gamma_n^2}{2},\]
      for all $n \ge N$. Then, claim (i) follows by applying Lemma~\ref{lem:as_rate_extension_beta_allgemein} with $c_1 =0, c_2 =\tilde c, c_3 = \frac{LC}{2}$ and $\beta = \frac 12$.\\
      To prove claim (ii) we multiply $(n+1)^{1-\eta}$ on both sides and take the expectation. It follows that
      \begin{align*}
          \E[(n+1)^{(1-\eta)}Y_{n+1}] &\leq (1-\tilde c \gg_n) (n+1)^{1-\eta}\E[Y_n] + \frac{LC}{2} (n+1)^{1-\eta} \gg_n^2\\
          &\leq (1-\tilde c \gg_n) (n^{1-\eta} +(1-\eta) n^{-\eta}) \E[Y_n]+ \frac{LC}{2} (n+1)^{1-\eta} \gg_n^2\\
          &= \left( 1- \tilde c \gg_n +\frac{1-\eta}{n} - \frac{\tilde c (1-\eta)\gg_n}{n}\right) n^{1-\eta}\E[Y_n]+ \frac{LC}{2} (n+1)^{1-\eta} \gg_n^2.
      \end{align*}
      As $\gt_n \in \Theta(\frac{1}{n^\gt})$ we obtain that $\tilde c \gg_n$ is the dominating term. Hence, there exists a constant $\tilde c_1>0$ and $\tilde N>N$ such that $\tilde c \gg_n -\frac{1-\eta}{n} + \frac{\tilde c (1-\eta)\gg_n}{n} \geq \tilde c_1 \gg_n$ for all $n\geq \tilde N$.
      Thus, for all $n\geq \tilde N$
      \begin{align*}
          \E[(n+1)^{(1-\eta)}Y_{n+1}] \leq\left( 1- \tilde c_1 \gg_n\right) n^{1-\eta}\E[Y_n]+ \frac{LC}{2} (n+1)^{1-\eta} \gg_n^2.
      \end{align*}
      We apply Lemma~\ref{lem:help-an-bn-cn} with $w_n= n^{1-\eta}\E[Y_n]$, $a_n = \tilde c_1 \gg_n$ and $b_n = (n+1)^{1-\eta} \gg_n^2$ and obtain that $n^{1-\eta}\E[Y_n] \to 0$ for $n \to \infty$ which yields claim (ii).
      Note that $\sum_n b_n <\infty$ as $1-\eta <2\gt-1$ for $\eta \in (2-2\gt,1)$.
     
     \underline{$\beta \in (\frac{1}{2},1]$:} In this case,~\cref{eq:smooth+ABC+gradient-dom} results in the super-martingale inequality 
      \[ \E[Y_{n+1} \mid \mathcal F_n] \leq \Big(1+\frac{LA\gamma_n^2}{2}\Big) Y_n - \Big(\gamma_n-\frac{BL\gamma_n^2}{2}\Big) c^2 Y_n^{2\beta} + \frac{LC\gamma_n^2}{2},\]
      with $Y_n = f(X_n)-f^\ast$.
      By the choice of $\gg_n$ there exists $c_2>0$ and $N_1>0$ such that $c^2\gamma_n-\frac{BLc^2\gamma_n^2}{2} \geq c_2 \gg_n$ for all $n\ge N_1$,
    \begin{align*}
        \E[Y_{n+1} \mid \mathcal F_n] 
        &\leq \Big(1+\frac{LA\gamma_n^2}{2}\Big) Y_n - c_2\gamma_n Y_n^{2\beta} + \frac{LC\gamma_n^2}{2}\,.
    \end{align*}
    We deduce claim (i) from Lemma~\ref{lem:as_rate_extension_beta_allgemein} with $c_1 = \frac{LA}{2}, c_2 = c_2, c_3 = \frac{LC}{2}$ and $\beta \in (\frac 12 ,1]$.
    
    For claim $(ii)$ we firstly proceed as in the proof of Lemma~\ref{lem:as_rate_extension_beta_allgemein}. 
    Therefore, one can choose the auxiliary parameter $1<q\leq \frac{1}{\gt}$ and find constants $c_4, c_3, \tilde N_1 >0$ such that for all $n \geq \tilde N_1$ by Equation\cref{eq:help-lem-21} we have
    \begin{equation*}
        \E[ (n+1)^{1-\eta} Y_{n+1}\mid \cF_n] \le n^{1-\eta} Y_n - c_4 \frac{1}{n^{q\theta}} n^{1-\eta} Y_n + c_3(n+1)^{1-\eta} (\gamma_n^\frac{2\beta q-1}{2 \beta-1}+\gamma_n^2)\,.
    \end{equation*}
    Next, we take the expectation to obtain
    \begin{align*}
        \E[ (n+1)^{1-\eta} Y_{n+1}] &\le (1 - c_4 \frac{1}{n^{q\theta}}) \E[n^{1-\eta} Y_n] + c_3(n+1)^{1-\eta} (\gamma_n^\frac{2\beta q-1}{2 \beta-1}+\gamma_n^2)
    \end{align*}
    for all $n\geq \tilde N_1$, implying that $w_n=\E[n^{1-\eta} Y_n] \to 0$ as $n\to\infty$ by Lemma~\ref{lem:help-an-bn-cn}. Note that we have chosen $\gt, \eta$ and $q$ as in Lemma~\ref{lem:as_rate_extension_beta_allgemein}, such that $\sum_n \frac{1}{n^{q\theta}} = \infty$, $\sum_n (n+1)^{1-\eta} \gamma_n^\frac{2\beta q-1}{2 \beta-1} < \infty$, and $\sum_n  (n+1)^{1-\eta} \gamma_n^2 <\infty$ (see \cref{eq:cond-1-lem21}, \cref{eq:cond-2-lem21} and \cref{eq:cond-3-lem21}). Therefore, the assumptions of Lemma~\ref{lem:help-an-bn-cn} are met.

\end{proof}

To the best of our knowledge our theorem presents the first convergence rate for SGD under weak gradient domination with respect to almost sure convergence. 

It is natural to ask which $\theta$ leads to the best convergence rate. 
Optimizing  for $\eta$ yields an optimal choice $\theta = \frac{2\beta}{4\beta-1}$ to achieve the best possible rate of convergence. This specific choice yields a lower bound of the interval given by $2-2\gt = \frac{\gt+2\beta-2}{2\beta-1} = 1-\frac{1}{4\beta-1}$ and therefore an almost sure convergence of the form $ o(\frac{1}{n^{p}})$ where $p$ is arbitrarily close to $\frac{1}{4\beta-1}$ (see also~\cref{tab:my-table}). We emphasize that the rate we obtain is arbitrarily close to the one obtained in \citet{pmlr-v134-fontaine21a, fatkhullin2022} in expectation. According to \citet[Prop. 2]{fatkhullin2022} the rate is attained, if the recursive inequality is indeed an equality.

Roughly speaking, our result guarantees a faster convergence rate for "stronger" gradient domination properties (i.e.~for smaller $\beta$). Indeed, as $2-2\theta>\frac{\gt+2\beta-2}{2\beta-1}$ for $\beta$ sufficiently close to $\frac12$ our result is consistent to the one presented in \citet[Thm. 1]{LY2022} by replacing the $\mu$-strongly convex assumption with the strong gradient domination property with $\beta = \frac{1}{2}$.

Similar arguments can be used to derive almost sure convergence rates for SHB under global gradient domination:

\begin{theorem}\label{thm:SHB-global}
Suppose Assumption~\ref{ass:L-smooth} \ref{ass:L-smooth-global} and Assumption~\ref{ass:ABC} 
are fulfilled and let $f$ satisfy the global gradient domination property from Definition~\ref{def:PL-defs} with $\beta\in[\frac12,1]$. Denote by $(X_n)_{n\in\N}$ the sequence generated by \cref{eq:SHB} using a step size $\gamma_n =\Theta( \frac{1}{n^\gt})$ for $\gt \in (\frac{1}{2},1)$.  For any 
    \[\eta \in\begin{cases} 
    \left(\max\{2-2\gt, \frac{\gt+2\beta-2}{2\beta-1}\},1\right) &:\beta\in(\frac12,1]\\
    (2-2\gt,1) &:\beta=\frac12 
    \end{cases}\] 
    it holds that
    \begin{equation*} 
    (i)\ \ f(X_n)-f^\ast \in o\left( \frac{1}{n^{1-\eta}}\right),\ \text{a.s.},\quad \text{and} \quad (ii)\ \ 
        \E[f(X_n) - f^\ast] \in o\left(\frac{1}{n^{1-\eta}}\right).
    \end{equation*}
\end{theorem}

For the proof of \Cref{thm:SHB-global}, recall the definition of the iteration scheme \cref{eq:SHB}.
Using the following definitions
\begin{equation}\label{eq:def-z_n-w_n}
    Z_n = X_n + \frac{\nu}{1-\nu} W_n, \quad W_n = X_n - X_{n-1},
\end{equation}
one can 
derive the iterative evolution 
\begin{align}\label{eq:SHB-rewritten}
    &W_{n+1} = \nu W_n - \gamma_n V(X_n)\\
    &Z_{n+1} = Z_n -\frac{\gamma_n}{1-\nu} V(X_n).
\end{align}
The first equation follows directly from the definition of $W_n$. For the second equation we compute 
\begin{align*}
    Z_{n+1} = X_{n+1} + \frac{\nu}{1+\nu} W_{n+1} &= (1+\frac{\nu}{1-\nu}) X_{n+1} - \frac{\nu}{1-\nu} X_n\\
    &=\frac{1-\nu}{1-\nu}X_n - \frac{\gamma_n}{1-\nu}V(X_n) + \frac{\nu}{1-\nu}(X_n-X_{n-1})\\
    &= Z_n - \frac{\gamma_n}{1-\nu}V(X_n)\,.
\end{align*}
We will utilize these auxiliary variables in the proof. 

\begin{proof}[Proof of \Cref{thm:SHB-global}]
    The proof begins as in the proof of Theorem 2 in \cite{LY2022}. Using only $L$-smoothness and assumption~\cref{eq:ABC-condition}, they show that for any $c_3 \in (0,\frac{1}{1-\nu})$, $\lambda\in(\nu,1)$ there exist constants $c_1,c_2,c_4>0$ such that choosing the step size $\gg_n\sim \frac{1}{n^{\gt}}$, for some $\gt \in(\frac 12,1)$ results in \citep[Equation (21)]{LY2022}
    \begin{align}\label{eq:SHB_eq_21}
    \begin{split}
        &\E[f(Z_{n+1})-f^\ast + \lVert W_{n+1}\rVert^2 \mid \cF_n] \\
        &\leq (1+c_1\gg_n^2) (f(Z_n)-f^\ast) + (\lambda+c_2\gg_n^2) \lVert W_n\rVert^2 - c_3 \gg_n \lVert \nabla f(Z_n)\rVert^2 +c_4 \gg_n^2
    \end{split}
    \end{align}
    for all $n\ge N$ and some $N>0$ sufficiently large.
    Next, we apply the global gradient domination property for any $\beta\in[\frac{1}{2},1]$ to derive
    \begin{align}\label{eq:SHB_eq_21+PL}
        \begin{split}
        &\E[f(Z_{n+1})-f^\ast + \lVert W_{n+1}\rVert^2 \mid \cF_n] \\
        &\leq (1+c_1\gg_n^2) (f(Z_n)-f^\ast) - c c_3 \gg_n (f(Z_n)-f^\ast)^{2\beta}+ (\lambda+c_2\gg_n^2) \lVert W_n\rVert^2 +c_4 \gg_n^2.
        \end{split}
    \end{align}  
    For the remaining proof, we denote $Q_n := f(Z_n)-f^\ast$. Similar as before, we treat both cases for $\beta=\frac12$ and $\beta\in(\frac12,1]$ separately:
    
    \underline{$\beta = \frac12$:}
    Instead of $\mu$-strong convexity we use the gradient domination inequality $\lVert \nabla f(x) \rVert^2 \geq c (f^\ast-f(x))$, as the same inequality is implied by strong convexity using $c=\mu$. Then, Claim (i) follows using the same proof as \citet[Thm. 2b)]{LY2022}. Note that the inequality 
    \begin{equation}\label{eq:help-SHB-global-proof}
        \frac{1}{2L} \lVert \nabla f(x) \rVert^2 \leq f(x)-f^\ast,
    \end{equation} 
    used in the last step only requires the $L$-smoothness assumption \citep[Sec. 1.2.3]{nesterov2}. 

    For Claim (ii) we consider~\cref{eq:SHB_eq_21+PL} which simplifies for $\beta = \frac 12$ to
    \begin{align*}
        &\E[Q_{n+1} + \lVert W_{n+1}\rVert^2 \mid \cF_n] \leq (1+c_1\gg_n^2- c c_3 \gg_n) Q_n+ (\lambda+c_2\gg_n^2) \lVert W_n\rVert^2 +c_4 \gg_n^2.
    \end{align*}
    By the choice of $\gg_n$ there exists $N>0$ and $\tilde c_1, \tilde c_2 >0$, such that $c c_3 \gg_n-c_1\gg_n^2 \geq \tilde c_1 \gg_n$ and $\lambda + c_2 \gg_n^2 \leq \tilde c_2 \gg_n$ for all $n \geq N$. Hence, for $n\geq N$
    \begin{align*}
        \E[Q_{n+1} + \lVert W_{n+1}\rVert^2 \mid \cF_n] 
        &\leq (1- \tilde c_1 \gg_n) Q_n+ (1-\tilde c_2\gg_n) \lVert W_n\rVert^2 +c_4 \gg_n^2 \\
        &\leq (1- \min\{\tilde c_1, \tilde c_2\}) \left(Q_n + \lVert W_n\rVert^2\right)+c_4 \gg_n^2.
    \end{align*}
    Let $c_5 = \min\{\tilde c_1, \tilde c_2\}$, multiply by $(n+1)^{1-\eta}$ on both sides and use~\cref{eq:elementary} to obtain for $n\geq N$
    \begin{align*}
        &\E[(n+1)^{1-\eta} \left(Q_{n+1} + \lVert W_{n+1}\rVert^2\right) \mid \cF_n] \\
        &\leq (n+1)^{1-\eta}(1- c_5\gg_n) \left(Q_n + \lVert W_n\rVert^2\right) +c_4 \gg_n^2(n+1)^{1-\eta} \\
        &\leq (n^{1-\eta}+ (1-\eta)n^{-\eta})(1- c_5) \left(Q_n + \lVert W_n\rVert^2\right)+c_4 \gg_n^2(n+1)^{1-\eta}\\
        &= \left( 1-c_5 \gg_n + \frac{1-\eta}{n}- \frac{c_5(1-\eta)\gg_n}{n}\right) n^{1-\eta} \left(Q_n + \lVert W_n\rVert^2\right)+c_4 \gg_n^2(n+1)^{1-\eta}.
    \end{align*}
    Taking expectation and using that there exists $\tilde c_5>0 $ and $\tilde N >N $ such that $c_5 \gg_n - \frac{1-\eta}{n}+ \frac{c_5(1-\eta)\gg_n}{n}\geq \tilde c_5 \gg_n$, we have for all $n\geq \tilde N$
    \begin{align*}
        \E[(n+1)^{1-\eta} \left(Q_{n+1} + \lVert W_{n+1}\rVert^2\right)]
        &\leq  \left( 1-\tilde c_5 \gg_n\right) \E\left[n^{1-\eta}\left(Q_n + \lVert W_n\rVert^2\right)\right]+c_4 \gg_n^2(n+1)^{1-\eta}.
    \end{align*}
    Note that $\sum_n \gg_n^2(n+1)^{1-\eta} <\infty$ because $\eta \in (2-2\gt,1)$ implies $1-\eta < 2\gt-1$. We can apply Lemma~\ref{lem:help-an-bn-cn} which yields that $\E\left[n^{1-\eta} \left(Q_n + \lVert W_n\rVert^2\right)\right] \to 0$. Hence, $\E[ \left(Q_n + \lVert W_n\rVert^2\right)] \in o(\frac{1}{n^{1-\eta}})$.\\
    To finish the proof, one can derive 
    \begin{equation}
        f(X_n)-f^\ast \leq Q_n + \frac{1}{2} \lVert \nabla f(Z_n)\rVert^2 + \frac{\nu^2+L \nu^2}{2(1-\nu)^2} \lVert W_n\rVert^2 \,.
    \end{equation}
    To derive this equation, recall that $Z_n-X_n= - \frac{\nu}{1-\nu}W_n$, such that by $L$-smoothness we obtain  
    \[f(X_n) \le f(Z_n) + \langle \nabla f(Z_n),X_n-Z_n\rangle + \frac{L}2 \|X_n-Z_n\|^2 = f(Z_n) - \frac{\nu}{1-\nu}\langle \nabla f(Z_n),W_n\rangle + \frac{L\nu^2}{2(1-\nu)^2} \|W_n\|^2.\]
    Next, apply Cauchy-Schwarz and Young's inequality to obtain
    \[f(X_n)-f^\ast \le f(Z_n)-f^\ast + \frac{1}{2}\|\nabla f(Z_n)\|^2 + \frac{\nu^2}{2(1-\nu)^2}\|W_n\|^2 + \frac{L\nu^2}{2(1-\nu)^2} \|W_n\|^2. \]
    Using inequality~\cref{eq:help-SHB-global-proof}, we get almost surely
    \begin{equation}\label{eq:help-SHB-2}
        f(X_n)-f^\ast \leq (1+L)Q_n  + \frac{\nu^2+L \nu^2}{2(1-\nu)^2} \lVert W_n\rVert^2 \le (1+L)\max\Big(1,\frac{\nu^2}{2(1-\nu)^2}\Big) (Q_n+\lVert W_n\rVert^2).
    \end{equation}
   implying that $\E[(f(X_n)-f^\ast) ] \in o(\frac{1}{n^{1-\eta}})$ which proves Claim (ii).

    \underline{$\beta \in (\frac12,1]$:} For Claim (i), note that in~\cref{eq:SHB_eq_21+PL} $\lambda<1$,
    such that
    \begin{align*}
        &\E[Q_{n+1}+ \lVert W_{n+1}\rVert^2 \mid \cF_n] \\
        &\leq (1+c_1\gg_n^2) Q_n + (1+ c_2 \gg_n^2) \lVert W_n\rVert^2+ c c_3 \gg_n  Q_n^{2\beta}  +c_4 \gg_n^2 \\
        &\leq (1+\max\{c_1,c_2\} \gg_n^2) (Q_n +\lVert W_n\rVert^2 ) + c c_3 \gg_n  (Q_n+\lVert W_n\rVert^2)^{2\beta}  +c_4 \gg_n^2.
    \end{align*}  
    By Lemma~\ref{lem:as_rate_extension_beta_allgemein} we obtain that $Q_n +\lVert W_n\rVert^2 = f(Z_n)-f^\ast +\lVert W_n\rVert^2 \in o\left( \frac{1}{n^{1-\eta}}\right)$ for all $\eta \in \left( \max\{2-2\gt,\frac{\gt+2\beta-2}{2\beta-1} \},1\right)$.
    We apply the inequality in \cref{eq:help-SHB-2} to conclude that also $f(X_n)-f^\ast \in o\left( \frac{1}{n^{1-\eta}}\right)$ for all $\eta \in \left( \max\{2-2\gt,\frac{\gt+2\beta-2}{2\beta-1} \},1\right)$.
    This proves Claim (i).

    For Claim (ii), we again use the $q$-trick from Lemma~\ref{lem:as_rate_extension_beta_allgemein} in~\cref{eq:SHB_eq_21+PL}. For $1<q<\frac{1}{\gt}<2$ 
    we have that
    \begin{align*}
        &\E[Q_{n+1} + \lVert W_{n+1}\rVert^2 \mid \cF_n] \\
        &\leq (1+c_1\gg_n^2- c c_3\gg_n^q) Q_n + c c_3 \gg_n \left( \gg_n^{q-1} Q_n - Q_n^{2\beta}\right) + (\lambda+c_2\gg_n^2) \lVert W_n\rVert^2 +c_4 \gg_n^2.
    \end{align*}
    Now with \cref{eq:function-trick} in Lemma~\ref{lem:as_rate_extension_beta_allgemein} there exists $\tilde c_3\ge0$ such that 
    \begin{align*}
        \E[Q_{n+1} + \lVert W_{n+1}\rVert^2 \mid \cF_n] \leq (1+c_1\gg_n^2- c c_3\gg_n^q) Q_n + \tilde c_3 \gg_n^{\frac{2\beta q-1}{2\beta-1}} + (\lambda+c_2\gg_n^2) \lVert W_n\rVert^2 +c_4 \gg_n^2.
    \end{align*}
    By the choice of $\gg_n$ there exists $\tilde c_1>0$ and $N>0$ such that $c_1\gg_n^2- c c_3\gg_n^q \geq \tilde c_1 \gg_n^q$ and $\lambda+c_2\gg_n^2 \leq \tilde c_1 \gg_n^q$ for all $n\geq N$. Thus, for all $n \geq N$,
    \begin{align*}
        &\E[Q_{n+1} + \lVert W_{n+1}\rVert^2 \mid \cF_n] \leq (1-\tilde c_1\gg_n^q) (Q_n+  \lVert W_n\rVert^2) + \max\{\tilde c_3,c_4\} \left(\gg_n^{\frac{2\beta q-1}{2\beta-1}} +\gg_n^2\right).
    \end{align*}
    For $\max\{\tilde c_3,c_4\} =:\tilde c_2$, we multiply on both sides with $(n+1)^{1-\eta}$ and take the expectation to obtain for $n \geq N$
    \begin{align*}
        &\E[(n+1)^{1-\eta}\left(Q_{n+1} + \lVert W_{n+1}\rVert^2\right)] \\
        &\leq (n+1)^{1-\eta} (1-\tilde c_1\gg_n^q) \E[(Q_n+\lVert W_n\rVert^2)] + \tilde c_2 (n+1)^{1-\eta}\left(\gg_n^{\frac{2\beta q-1}{2\beta-1}} +\gg_n^2\right)\\
        &\leq (n^{1-\eta}+(1-\eta) n^{-\eta}) (1-\tilde c_1\gg_n^q) \E[(Q_n+ \lVert W_n\rVert^2)] + \tilde c_2(n+1)^{1-\eta} \left(\gg_n^{\frac{2\beta q-1}{2\beta-1}} +\gg_n^2\right)\\
        &= \left(1- \tilde c_1 \gg_n^q + \frac{1-\eta}{n} - \frac{\tilde c_1 (1-\eta) \gg_n^q}{n}\right) \E[n^{1-\eta} (Q_n+\lVert W_n\rVert^2)]\\
        &\quad+ \tilde c_2 (n+1)^{1-\eta}\left(\gg_n^{\frac{2\beta q-1}{2\beta-1}} +\gg_n^2\right).
    \end{align*}
    Next, there exists $\tilde N >N$ and $\tilde c_5>0$ such that for all $n\geq \tilde N$
    \begin{align*}
        &\E[(n+1)^{1-\eta}\left(Q_{n+1} + \lVert W_{n+1}\rVert^2\right)] \\
        &\leq \left(1- \tilde c_5 \gg_n^q \right) \E[n^{1-\eta} (Q_n+\lVert W_n\rVert^2)]+ \tilde c_2 (n+1)^{1-\eta}\left(\gg_n^{\frac{2\beta q-1}{2\beta-1}} +\gg_n^2\right).
    \end{align*}
    From the proof of Lemma~\ref{lem:as_rate_extension_beta_allgemein}, we choose the auxiliary parameter $q$ such that 
    \[ \sum_n (n+1)^{1-\eta}\left(\gg_n^{\frac{2\beta q-1}{2\beta-1}} +\gg_n^2\right) <\infty\,,\] 
 see \cref{eq:cond-2-lem21} and \cref{eq:cond-3-lem21}. By applying again Lemma~\ref{lem:help-an-bn-cn} we obtain $\E[n^{1-\eta} (Q_n+\lVert W_n\rVert^2)] \to 0$, i.e. $\E[ Q_n+\lVert W_n\rVert^2] \in o(\frac{1}{n^{1-\eta}})$. Finally, Claim (ii) follows again by~\cref{eq:help-SHB-2}.
\end{proof}

To the best of our knowledge, our result gives the first convergence proof of SHB to global optima under weak gradient domination, with rates for almost sure convergence and convergence of expectations. 
The resulting convergence rate using the optimized step size are summarized in~\cref{tab:my-table}. In the strong gradient domination setting our rate in expectation gets arbitrarily close to the $O(\frac{1}{n})$ convergence obtained in \citet{LIANG2023830}. 
It is noteworthy that the utilization of SHB in our analysis does not yield a superior convergence rate compared to SGD. This arises from the proof technique and aligns with the findings in \citet{LY2022, pmlr-v134-sebbouh21a} where the authors similarly achieve no acceleration. 
In general, for deterministic settings acceleration of gradient methods can achieve improvements of convergence rates \citep{Wilson2019AcceleratingRG}. Though in the special case of gradient domination with $\beta=\frac 12$, HB as well as Nesterov cannot accelerate in the deterministic setting as shown in \citet{pengyun-2023}.

\section{Convergence for Local Gradient Domination Property}\label{sec:localSGD}

In this section, we want to generalize the analysis in \cite{mertikopoulos2020sure} under local strong convexity to the weaker local gradient domination property for different cases of $\beta$. We consider the two cases of local gradient domination separately. 
The contributions and differences of our results 
under less restricted assumptions are the following:\\
First, we show in both cases that SGD remains in the gradient dominated region 
with high probability by only assuming local gradient domination instead of local strong convexity. Especially in the case of a local minimum $x^\ast$ this is a challenging task, as we have to ensure that the SGD scheme $(X_n)_{n\in\N}$ remains close to $x^\ast$ without exploiting convexity. We can guarantee this whenever $x^\ast$ is in an isolated connected compact set of local minima $\cX^\ast$. We prove convergence towards the level set of $\cX^\ast$ and obtain \cref{thm:conv-SGD-local-PL-xast}. 
Second, additionally to convergence in expectation we prove almost sure convergence conditioned on the "good event".
Third, due to the weaker gradient domination assumption and no convexity, one cannot expect the convergence of $X_n$ to (local or global) minimum $x^\ast$, instead we focus on convergence of $f(X_n)$ to $f(x^\ast)$. In \citet{liuZhou2023} they delve into the rationale behind considering this as a more robust metric.

The main result under local gradient domination in a local minimum $x^\ast$ is as follows:
\begin{theorem}\label{thm:conv-SGD-local-PL-xast}
    Fix some tolerance level $\delta >0$ and let $\cX^\ast\subset \R^d$ be an isolated compact connected set of local minima with level $l = f(x^\ast)$ for all $x^\ast \in\cX^\ast$. Suppose that $f$ satisfy the local gradient domination property in each $x^\ast\in\cX^\ast$, $f$ is locally $G$-Lipschitz continuous and satisfies Assumption~\ref{ass:L-smooth} \ref{ass:L-smooth-local}. Moreover, suppose Assumption~\ref{ass:ABC} hold true. Denote by $(X_n)_{n\in\N}$ the sequence generated by \cref{eq:SGD} using a step size $\gamma_n = \Theta(\frac{1}{n^{\theta}})$ for $ \theta\in (\frac{1}{2}, 1)$ and suppose that $\gg_n \leq \gamma_1$ for $\gamma_1$ sufficiently small enough such that $\sum_{n=1}^{\infty} \gg_n^2 < \frac{\delta \epsilon}{2(G^2 C^2 + G^2 + C)}$ for some $\epsilon>0$ independent of $\delta$. 
    Then, the following holds:
    \begin{enumerate}
        \item[(i)] There exist subsets $\cU$ and $\cU_1$ of $\R^d$ such that, if $X_1 \in\cU_1$ the event $\Omega_{\cU} = \{X_n \in\cU \textrm{ for all } n=1,2,\dots \}$ has probability at least $1-\delta$.
    \end{enumerate}
    Moreover, there exists $\beta \in [\frac 12, 1]$ such that for any 
    \[\eta \in\begin{cases} 
    \left(\max\{2-2\gt, \frac{\gt+2\beta-2}{2\beta-1}\},1\right) &: \beta\in(\frac12,1]\\
    (2-2\gt,1) &:\beta=\frac12 
    \end{cases}\] 
    it holds that
        \begin{equation*}
        (ii)\ \ |f(X_n) - l |\mathbf{1}_{\Omega_{\cU}} \in o\left(\frac{1}{n^{1-\eta}}\right),\  \text{a.s.},\quad \text{and}\quad
        (iii)\ \ \E[|f(X_n) - l | \mathbf{1}_{\Omega_{\cU}}] \in o\left(\frac{1}{n^{1-\eta}}\right).
        \end{equation*}
\end{theorem}

We would like to emphasize that our result does not provide convergence rates in high probability. Instead, we restrict the probability space to "good" events (encoded in $\mathcal U$) that occur with high probability $1-\delta$. Conditioned on these events, we establish convergence with a rate that is independent of $\delta$. The dependence on $\delta$ instead manifests in the step size $\gamma_n$, which ensure that the iterations $(X_n)_{n\in\N}$ remain within $\cU$ with probability at least $\delta$. Notably, this is the only aspect of our convergence result that depends on $\delta$.

In the following, we only sketch the proof in the main part and 
provide the full proof 
in \cref{app:proof-localSGD}.
\begin{proof}[Sketch of proof]
   The proof is split into several intermediate steps outlined as follows:
\begin{itemize}
   \item First, we unify the gradient domination property around the set of local minima $\cX^\ast$ and obtain a radius $\mathbf r$ such that the unified gradient domination property is fulfilled in all open balls with radius $\mathbf r$ around $x^\ast\in\cX^\ast$ (Lemma~\ref{lem:unify-gradientdom-xast}). 
   \item Based on this we construct the two sets $\cU$, $\cU_1 \subseteq \R^d$ defined as neighborhoods of $\cX^\ast$ constructed such that the gradient domination property holds within this region,
   \begin{align*}
       \cU_1 &= \{ x\in\R^d\,: \, \inf_{x^\ast\in\cX^\ast}||x-x^\ast|| < \frac{\mathbf{r}}{2}, f(x)-l \leq \frac{\epsilon}{2}\}\\
       \cU &= \{  x\in\R^d\,: \, \inf_{x^\ast\in\cX^\ast}||x-x^\ast|| < \frac{\mathbf{r}}{2} \}
   \end{align*}
   and the events
   \[\Omega_n = \{X_k \in \cU \textrm{ for all } k \leq n\}\in\Omega \]
   such that $\Omega_\cU = \bigcap_{n} \Omega_n$ occurs with high probability. 
   This means, when starting in $\cU_1$ the gradient trajectory does remain in $\cU$ for all gradient steps with high probability.  
   \item We present a row of Lemmata to show that $\bbP(\Omega_n) \geq 1-\delta$ for all $n \in\N$, claim (i) of the Theorem. To show $\bbP(\Omega_n) \geq 1-\delta$ we construct the sets $C_n$ and $E_n$ defined in \cref{eq:h4} and \cref{def:E_n-x-ast} such that $E_n \cap C_n \subset \Omega_{n+1}$ (Lemma~\ref{lem:tilde-r-n-xast}) while Lemma~\ref{lem:aux-lemma-x_n-close-to-xast} is used to prove this claim. The sets $E_n$ are such that $f(X_n)$ remains close to $f^\ast$. We exploit the unified gradient domination property to construct the sets $E_n$ (Lemma~\ref{lem:bar_D_iteration_x_ast}) and derive a recursive inequality in Lemma~\ref{lem:tilde-r-n-xast} c) to prove that this event occurs with high probability (Lemma~\ref{lem:prob_E_n_lager_1-delta-xast}). The sets $C_n$ are such that $X_{n+1}$ remains close to $X_n$ and we exploit the finite variance assumption to show that these events occur with high probability (Lemma~\ref{lem:prob_hat-E_n_lager_1-delta}). 
   \item Finally, claim (ii) and (iii) are shown directly in the proof of \cref{thm:conv-SGD-local-PL-xast} at the end of \cref{app:proof-localSGD}, where we employ the uniform gradient domination in $\Omega_{\mathcal U}$.
\end{itemize}
\end{proof}

The main result concerning local gradient domination in $f^\ast$ is presented below and does not necessitate the existence of a local minimum or any stationary point. It is worth noting that the definition of local gradient domination in $f^\ast$ guarantees the gradient domination property for any $x$ with $f(x)$ close to $f^\ast$. Consequently, this definition ensures that functions satisfying this property cannot possess local minima or saddle points within this region. 
\begin{theorem}\label{thm:conv-SGD-local-PL-fast}
     Fix some tolerance level $\delta >0$. Suppose $f$ satisfies the local gradient domination property in $f^\ast$ from Definition~\ref{def:PL-defs} with $\beta\in[\frac12,1]$ and $\cB_{r}^\ast\subseteq\R^d$. Moreover, suppose within $\cB_{r}^\ast$ $f$ is $G$-Lipschitz continuous, Assumption~\ref{ass:L-smooth} \ref{ass:L-smooth-global} and Assumption~\ref{ass:ABC} hold true. Denote by $(X_n)_{n\in\N}$ the sequence generated by \cref{eq:SGD} using a step size $\gamma_n = \Theta(\frac{1}{n^{\theta}})$ for $ \theta\in (\frac{1}{2}, 1)$ and suppose that $\gg_n \leq \gamma_1$ for $\gamma_1$ 
     sufficiently small such that $\sum_{n=1}^{\infty} \gg_n^2 < \frac{\delta \epsilon}{2(G^2 C^2 + G^2 + C)}$ for some $\epsilon>0$ independent of $\delta$. Then, the following holds:
    \begin{enumerate}
        \item[(i)] There exist subsets $\cU$ and $\cU_1$ of $\R^d$ such that, if $X_1 \in\cU_1$ the event $\Omega_{\cU} = \{X_n \in\cU \textrm{ for all } n=1,2,\dots \}$ has probability at least $1-\delta$.
    \end{enumerate}
    Moreover, for any 
    \[\eta \in\begin{cases} 
    \left(\max\{2-2\gt, \frac{\gt+2\beta-2}{2\beta-1}\},1\right) &: \beta\in(\frac12,1]\\
    (2-2\gt,1) &:\beta=\frac12 
    \end{cases}\] 
    it holds that
        \begin{equation*}
        (ii)\ \ (f(X_n) - f^\ast ) \mathbf{1}_{\Omega_{\cU}} \in o\left(\frac{1}{n^{1-\eta}}\right),\ \text{a.s.}, \quad \text{and}\quad
        (iii)\ \ \E[(f(X_n) - f^\ast ) \mathbf{1}_{\Omega_{\cU}}] \in o\left(\frac{1}{n^{1-\eta}}\right).
        \end{equation*}
\end{theorem}

The outline of the proof is similar to the proof of \Cref{thm:conv-SGD-local-PL-xast}. We provide the full proof 
and more details on the sets $\cU$ and $\cU_1$ as well as on $\beta$ in \cref{app:proof-localSGD}.

\section{Application in the training of neural networks}
In supervised learning one aims to approximate an unknown model $\varphi:\R^{d_z}\to \R^{d_y}$ by a parametrized function $g_w:\R^{d_z}\to\R^{d_y}$ with parameter $w\in\R^{d_w}$. Given a family of training data $( (Z^{(m)},Y^{(m)}))_{m\in\N}$ generated as i.i.d.~samples from 
an unknown distribution $\mu_{(Z,Y)}$ one usually chooses the parameter $w\in\R^{d_w}$ by solving 
\begin{equation*}
    \min_{w\in\R^{d_w}}\ \mathbb E_{\mu_{(Z,Y)}}[\Phi(g_w(Z),Y)]\,,
\end{equation*}
where $\Phi:\R^{d_y}\times \R^{d_y}\to\R_+$ is a user specific data discrepancy. One popular choice of parametrizations are DNNs. We define a neural network of depth $L\in\N$ by the recursion
\[z_0:=z,\quad z_\ell = \sigma^{\otimes d_\ell}(A_\ell z_{\ell-1} + b_\ell),\ \ell=1,\dots,L-1,\quad g_w(z):= A_L z_{L-1} + b_{L}\,.\]
The weights $((A_\ell,b_\ell))_{\ell=1}^L$ of the DNN are collected in $w\in\cW:= \times_{\ell=1}^L (\R^{d_\ell\times d_{\ell-1}}\times \R^{d_\ell}) \simeq \R^{d_w}$, and $\sigma^{\otimes d}:\R^d\to\R^d$ describes the component-wise application of the activation function $\sigma:\R\to\R$. 

Provided that $\sigma$ and $\Phi$ are analytic, and $(Z,Y)$ are compactly supported $\R^{d_z}\times \R^{d_y}$-valued random variables,
then $f^{\mathrm{DNN}}:\R^{d_w}\to\R_+$ defined by $w\mapsto \E_{\mu_{(Z,Y)}}[\Phi(g_w(Z),Y)]$ is analytic \cite[Thm.~5.2]{dereich2021convergence}. 
By \citet[Thm.~II]{lojasiewicz1963propriete}, or \citet[§2, Thm.~2]{lojasiewicz1965}, it therefore satisfies local gradient domination in any stationary point. More precisely, for any stationary point $w_\ast$ there exist $\beta_{w_\ast} \in [\frac 12, 1]$ and $c_{w_\ast}>0$ such that \eqref{eq:loj} in Definition~\ref{def:PL-defs} is satisfied.
The local smoothness is inherently guaranteed by the fact that the loss function is analytic. To elaborate, every analytic function is infinitely differentiable, such that all derivatives remain locally bounded. Concerning Assumption~\ref{ass:ABC}, we assume that the underlying data in the supervised learning problem is compactly supported, ensuring that Assumption~\ref{ass:ABC} is always satisfied locally.

In our notation, the stochastic first order oracle takes the form 
\[V(w,(Z,Y)) = \nabla_w f^{\mathrm{DNN}}(w) + (\nabla_w \Phi(g_w(Z),Y)-\nabla_w f^{\mathrm{DNN}}(w))\,,\]
where we denote $\zeta = (Z,Y)$ and the iterative SGD then reads as 
\[W_{n+1} = W_n - \gamma_n \nabla_w \Phi(g_{W_n}(Z_{n+1}),Y_{n+1})\]
with $\zeta_n=(Z_n,Y_n)$ independent and identical distributed. The iterative scheme of SHB can be written similarly. Note that this scenario also includes the empirical risk minimization of $\frac{1}M\sum_{m=1}^M \Phi(g_w(z^{(m)}),y^{(m)})$ when $\zeta =(Z,Y)\sim \frac{1}M\sum_{m=1}^M \delta_{(z^{(m)},y^{(m)})}$, see \cref{ex:example} for more details. In this case, Assumption~\ref{ass:ABC} is satisfied even with $A = B = 0$. Thus, the following local convergence is a direct consequence of Theorem~\ref{thm:conv-SGD-local-PL-xast}.

\begin{corollary}\label{cor:DNN}
    Let $\delta >0$. Denote by $(W_n)_{n\in\N}$ the sequence generated by SGD with $w \mapsto \nabla_w f^{\mathrm{DNN}}(w)$ as objective function, step size $\gg_n \in \Theta(n^{-\gt})$ for $\gt\in(\frac{1}{2},1)$, and assume that $f^{\mathrm{DNN}}$ is analytic. Let $\mathcal W^{\ast}$ be an isolated compact set  of local minima with level $l=f^{\mathrm{DNN}}(w^\ast)$ for all $w^\ast\in \mathcal W^\ast$ and suppose Assumption~\ref{ass:ABC} is satisfied within $\mathcal W^{\ast}$. Suppose that $\gg_n\le \gg_1$ for sufficiently small $\gamma_1$ (depending on $\delta$), then there exist two subsets $\cU, \cU_1$ of $\R^{d_w}$ such that $W_1 \in \cU_1$ implies that the event $\Omega_\cU =\{W_n \in \cU, \textrm{ for all } n\geq 1\}$ has  probability at least $1-\delta$. Moreover, there exists $\beta \in [\frac 12, 1]$ such that for any 
    \[\eta \in\begin{cases} 
    \left(\max\{2-2\gt, \frac{\gt+2\beta-2}{2\beta-1}\},1\right) &: \beta\in(\frac12,1]\\
    (2-2\gt,1) &:\beta=\frac12 
    \end{cases}\] 
    it holds that $|f^{\mathrm{DNN}}(W_n) - l | \mathbf{1}_{\Omega} \in o\big( n^{\eta-1}\big)$ almost surely and in expectation.
\end{corollary}
In words: If the iterates of SGD reach a certain area around a local minimum, they are likely to become trapped in that region with high probability, provided that the step size is sufficiently small. 
This results shows that, under very general conditions, SGD converges to local minima and furthermore quantifies the convergence speed.

\begin{remark}
    One may similarly apply \cref{thm:conv-SGD-local-PL-fast} in the training of DNNs to derive convergence towards a global minimum with high probability provided that the initial loss $f^{\mathrm{DNN}}(X_1)$ and initial step size $\gamma_1$ are sufficiently small.
\end{remark}

\section{Application in Reinforcement Learning}\label{sec:application-RL}
Recent results showed that choosing the tabular softmax parametrization in policy gradient (PG) algorithms results in objective functions which fulfill a non-uniform gradient domination property \citep{mei2020b,yuan2022general, klein2023stationarity}. While convergence of PG for exact gradients is well understood, convergence rates for stochastic PG are rare and mostly require very large batch sizes \citep{ding2022beyondEG,klein2023stationarity,ding2023local}. \medskip

Let $(\cS,\cA,\rho,r,p)$ be a discounted MDP with finite state space $\cS$, finite action space $\cA$ and discount factor $\rho\in [0,1)$.  
Further, $r:\cS \times \cA \to \R_+$ is the positive expected reward function and $p(s^\prime |s,a)$ describes the transition probability from state $s$ to $s^\prime$ under action $a$. As in \citet{mei2020b} we assume that the rewards are bounded in $[0,1]$. Consider the stationary tabular softmax policy for parameter $w\in\R^{|\cS||\cA|}$, i.e.
\[ \pi_w(a|s) = \frac{\exp(w(s,a))}{\sum_{a^\prime \in\cA_{s}}\exp(w(s,a^\prime))}, \quad \forall s \in \cS,\, a\in\cA.\]
In the following we consider entropy regularized PG jointly with vanilla PG by setting $\lambda =0$ when no regularization is considered.
Then, for an initial state distribution $\mu$, the value function under the softmax parametrization is given by 
  \[  V_\lambda^{\pi_w}(\mu) =  \E_\mu^{\pi_w}\Big[ \sum_{t=0}^\infty \rho^t r(S_t,A_t) \Big] - \lambda \E_\mu^{\pi_w}\Big[ \sum_{t=0}^\infty \rho^t \log(\pi_w(A_t|S_t) \Big] \]
and denote by $V_\lambda^\ast(\mu)$ the global optimum and by $\pi^\ast$ the optimal policy. If $\lambda >0$ there exists a unique optimal policy which can be represented by the tabular softmax parametrization \cite{nachum2017bridging}. Therefore, there must exist a continuum of optimal parameters $w^\ast$, such that $V_\lambda^{\pi_{w^\ast}}(\mu) = V_\lambda^\ast(\mu)$; see also the discussion in \cite{ding2022beyondEG}. If $\lambda =0$ no such parameters exists.

In order to maximize the objective we use stochastic gradient ascent. 
For $\lambda=0$ this is called stochastic policy gradient method or REINFORCE. Note that e.g., \citet{Zhang2020} provide a stochastic first order oracle which meets the conditions required in Assumption~\ref{ass:ABC}. 
For $\lambda >0$, a stochastic gradient estimator that satisfies Assumption~\ref{ass:ABC} with $A=B=0$ is presented in \citet[Eq. (4)]{ding2023local}. 
In both cases a non-uniform PL-inequality holds \citep[Lem. 8, Lem. 15]{mei2020b}: For every $w \in \R^{|\cS|\times|\cA|}$ it holds that 
\begin{align*}
    \lVert \nabla_w V_\lambda^{\pi_w}(\mu)\rVert_2 \geq c_\lambda(w)^{x} \left[ V_\lambda^\ast(\mu) -V_\lambda^{\pi_w} (\mu)\right]^{x},
\end{align*}
with $x=1$ if $\lambda =0$ and $x= \frac 12$ if $\lambda >0$ and
\begin{align}
    c_\lambda(w) = \begin{cases} \frac{\min_{s\in\cS} \pi_w(a^\ast(s)|s)}{\sqrt{|\cS|}(1-\rho) }\Big\lVert\frac{d_\mu^{\pi^\ast}}{\mu} \Big\rVert_\infty^{-1}, &\lambda=0, \\
    \frac{2\lambda}{|\cS|(1-\rho)} \min_s \mu(s) \min_{s,a} \pi_w(a|s)^2 \Big\lVert\frac{d_\mu^{\pi^\ast}}{\mu} \Big\rVert_\infty^{-1}, &\lambda>0.
\end{cases}
\end{align}
Here $a^\ast(s)$ denotes the best possible action in state $s$. W.l.o.g. we assume that $a^\ast(s)$ is unique, otherwise one can simply consider the maximum over all possible best actions, i.e. replace $\min_{s\in\cS} \pi_w(a^\ast(s)|s)$ with $\min_{s\in\cS}\min_{a \textrm{ is optimal action in } s} \pi_w(a|s)$. 
We prove that this implies a local gradient domination property with $\beta=1$ (weak PL) for $\lambda =0$ and $\beta=\frac 12$ (strong PL) for $\lambda>0$. 
\begin{proposition}\label{lem:localPL}
    There exists $r,c>0$ such that for all $w\in \cB_{r,\lambda}^\ast = \{w: V_\lambda^\ast(\mu) - V_\lambda^{\pi_{w}}(\mu) \leq r \}$ it holds that $c_\lambda(w) \geq c$. 
\end{proposition}
The proof of this proposition is given in \Cref{app:proof-localPL_RL}.

As the objective function $w \mapsto V_\lambda^{\pi_w} (\mu) $ is smooth and Lipschitz  on $\R^{|\cS||\cA|}$ (see \citet[Lem. E.1]{yuan2022general} for $\lambda =0$ and \citet{ding2023local} for $\lambda>0$), all assumptions in Theorem~\ref{thm:conv-SGD-local-PL-fast} are satisfied and we obtain the following result.
\begin{corollary}\label{cor:discounted-convergence}
    Let $\delta >0$. Denote by $(W_n)_{n\in\N}$ the sequence generated by SGD with $w \mapsto-V_\lambda^{\pi_w}(\mu)$ as objective function, step size $\gg_n \in \Theta(n^{-\gt})$ for $\gt\in(\frac{1}{2},1)$ and suppose $\gg_n\le \gg_1$ for sufficiently small $\gamma_1$ (depending on $\delta$). Then, there exist two subsets $\cU, \cU_1$ of $\R^{|\cS||\cA|}$ such that $W_1 \in \cU_1$ implies that the event $\Omega_\cU =\{W_n \in \cU, \textrm{ for all } n\geq 1\}$ has  probability at least $1-\delta$. Moreover, for any 
    \begin{equation*} 
    \eta \in\begin{cases}
        (\max\{2-2\gt,\gt\},1),& \text{if}\ \lambda=0, \\ (2-2\gt,1), &\text{if}\ \lambda>0
    \end{cases}
    \end{equation*}
    it holds that $(V_\lambda^\ast(\mu) - V_\lambda^{\pi_{W_n}}(\mu) ) \mathbf{1}_{\Omega} \in o\big( n^{\eta-1}\big)$  almost surely and in expectation.
\end{corollary}
In words: If the (regularized) stochastic policy gradient algorithm is started close enough to the optimum a nearly $o(n^{-\frac{1}{3}})$ ($o(n^{-1})$ respectively) almost sure rate of convergence can be obtained by choosing $\gt= \frac{2}{3}$ ($\gt$ close to $1$ respectively) This is in contrast to $o(n^{-1})$ (linear convergence) known in (regularized) policy gradient with access to exact gradients.
\begin{remark}
    Note that $r$ and $c$ in Lemma~\ref{lem:localPL} can be explicitly chosen (see Remark~\ref{rem:c-r-depend-on-alpha}). Hence, one can choose the neighbourhoods $\cU$ and $\cU_1$ w.r.t. $r$ as in \cref{eq:B_epsilon_ast} and Lemma~\ref{lem:tilde-r-n} in \cref{app:proof-localSGD} to find an explicit neighbourhood $\cU_1$ as condition for initialization.
\end{remark}

Our convergence result in \Cref{cor:discounted-convergence} extends the local convergence analysis of stochastic policy gradient under entropy regularization from \citet{ding2023local}. Like their work, our analysis builds on \citet{mertikopoulos2020sure}. However, while \citet{ding2023local} require an increasing batch size sequence to ensure convergence to the regularized optimum, we achieve local convergence for both the unregularized and entropy-regularized settings without this requirement. For $\lambda > 0$, our method attains convergence arbitrarily close to $o(\frac{1}{n})$ that matches the rate obtained by \citet{ding2023local} while avoiding the need to increase batch sizes. Finally, we highlight that our local convergence holds almost surely on an event with high probability.


\paragraph{Acknowledgements} Sara Klein thankfully acknowledges the funding support by the Hanns-Seidel-Stiftung e.V. and is grateful to the DFG RTG1953 ”Statistical Modeling of Complex Systems and Processes” for funding this research.
Waïss Azizian was supported in part by the French National Research Agency (ANR) in the
framework of the PEPR IA FOUNDRY project (ANR-23-PEIA-0003) and MIAI@Grenoble Alpes (ANR-19-P3IA-0003).

\bibliography{SGD-paper}
\bibliographystyle{plainnat}


\newpage

\appendix

\section{Auxiliary Convergence Theorems}
In the following section, we provide two specific convergence theorems used to prove almost sure convergence (Lemma~\ref{cor:RS_corollary}) as well as convergence in expectation (Lemma~\ref{lem:help-an-bn-cn}). The former one is a direct consequence of the well-known Robbins-Siegmund theorem, provided here for completeness. 
\begin{theorem}[Theorem~1 in \cite{RS1971}]\label{app:thm:RS}
    Let $(\Omega,\mathcal F,(\mathcal F_n)_{n\in\mathbb N},\mathbb P)$ be a filtered probability space, $(Z_n)_{n\in\mathbb N}$, $(A_n)_{n\in\mathbb N}$, $(B_n)_{n\in\mathbb N}$ and $(C_n)_{n\in\mathbb N}$ be non-negative and adapted stochastic processes with
    \[\sum_{n=1}^\infty A_n<\infty \quad \text{and} \quad \sum_{n=1}^\infty B_n <\infty \]
almost surely. Suppose that for each $n\in\mathbb N$ the recursion
    \[ \mathbb E[Z_{n+1} \mid \mathcal F_n]\le (1+A_n) Z_n + B_n - C_n\]
    is satisfied, then (i) there exists an almost surely finite random variable $Z_\infty$ such that $Z_n \to Z_\infty$ almost surely as $n\to\infty$ and (ii) $\sum_{n=1}^\infty C_n<\infty$ almost surely.
\end{theorem}

\begin{lemma}\label{cor:RS_corollary}
Let $(\Omega,\mathcal F,(\mathcal F_n)_{n\in\mathbb N},\mathbb P)$ be a filtered probability space, $(Y_n)_{n\in\mathbb N}$, $(a_n)_{n\in\mathbb N}$, $(b_n)_{n\in\mathbb N}$ and $(r_n)_{n\in\mathbb N}$ be non-negative and adapted stochastic processes with 
\[\sum_{n=1}^\infty a_n = \infty, \quad \sum_{n=1}^\infty b_n <\infty \quad \text{and}\quad r_n>0  \]
almost surely. Suppose that for each $n\in\mathbb N$ the recursion
\[
\mathbb E[r_{n+1} Y_{n+1} \mid \mathcal F_n] \le (1-a_n)r_n Y_n + b_n
\]
is satisfied, then we have $r_n Y_n \to 0$ almost surely as $n\to \infty$. 
\end{lemma}
\begin{proof}
    We define $Z_n:= r_n Y_n$, $B_n:=b_n$ and $C_n := a_n r_n Y_n$ such that
    \[ \mathbb E[Z_{n+1}\mid \mathcal F_n] \le Z_n - C_n + B_n\]
    for $n\in\mathbb N$. Using Lemma~\ref{app:thm:RS} we observe that there exists $Z_\infty$ almost surely finite such that $Z_n = r_nY_n \to Z_\infty$ almost surely as $n\to\infty$. Moreover, we obtain that
    \[\sum_{n=1}^\infty C_n = \sum_{n=1}^\infty a_n r_n Y_n <\infty\]
    almost surely, which yields that 
    \[\liminf_{n\to\infty} r_nY_n = 0\]
    almost surely, since $\sum_{n=1}^\infty a_n = \infty$ almost surely. Since limit inferior and limit coincide for converging sequences, the assertion follows:
    \[Z_\infty = \lim_{n\to\infty} r_n Y_n = \liminf_{n\to\infty} r_n Y_n = 0 \]
    almost surely.
\end{proof}

The following Lemma will be applied to prove convergence in expectation.
\begin{lemma}\label{lem:help-an-bn-cn}
    Let $(w_n)_{n\in\N}$ be a non-negative sequence, such that $w_{n+1} \leq (1-a_n) w_n + b_n$, where $(a_n)_{n\in\N}$ and $(b_n)_{n\in\N}$ are non-negative sequences satisfying
    \begin{align*}
        \sum_{n=1}^\infty a_n = \infty \quad \textrm{ and } \quad \sum_{n=1}^\infty b_n <\infty.
    \end{align*}
    Then, $\lim_{n\to\infty} w_n =0$.
\end{lemma}

\begin{proof}
    W.l.o.g we assume that $w_{n+1} = (1-a_n) w_n + b_n$, otherwise we could just increase $a_n$ or decrease $b_n$ which would have no effect on the summation tests. 
    We obtain
    \begin{align*}
        -w_1 \leq w_n -w_1 = \sum_{k=1}^{n-1} (w_{k+1} - w_k) =  \sum_{k=1}^{n-1} b_k - \sum_{k=1}^{n-1} w_k a_k.
    \end{align*}
    Since $w_n - w_1$ is bounded below and $\sum_{k=1}^{\infty} b_k <\infty$, we deduce that $\sum_{k=1}^{n} w_k a_k$ is bounded. Since all summands are positive, the infinite sum converges. Thus, as a difference of two converging series also $(w_n)_{n\in\N}$ converges. Finally, the convergence of $\sum_{k=1}^{\infty} w_k a_k$ implies $\liminf_{n\to\infty}w_n=0$ which, by the convergence of $(w_n)_{n\in\N}$, implies $\lim_n w_n=\liminf_n w_n=0$.
\end{proof}

\section{Numerical experiment - Toy example}\label{app:numerics}
In the following numerical experiment, we aim to verify our theoretical finding. We have implemented the same toy example similar to \cite{fatkhullin2022} to test our theoretical findings. In our implementation, we consider both SGD and SHB applied to the objective function $f_p(x) = |x|^p$, where $x\in\mathbb{R}$, for various choices of $p\geq 2$. It is straightforward to verify that $f_p$ satisfies the global gradient domination with parameter $\beta(p) = \frac{p-1}{p}$. It is noteworthy that for $p=2$, the $f_p$ obviously satisfies the PL condition with $\beta=\frac12$, whereas for increasing $p\to\infty$, we move towards $\beta(p)\to1$. We have used the step size schedule $\Theta(n^{-\frac{2\beta(p)}{4\beta(p)-1}})$ discussed in Table~\ref{tab:my-table} and observed the almost sure convergence rates $n^{-\frac{1}{4\beta(p)-1}}$ as suggested by \cref{thm:global_convergencerate_SGD_beta_allgemein} and \cref{thm:SHB-global}. Note that our derived rates are arbitrarily close to the sharp upper bound known in expectation \cite{fatkhullin2022}.

\begin{figure}[htb!]
\begin{center}
\includegraphics[width=0.7\textwidth]{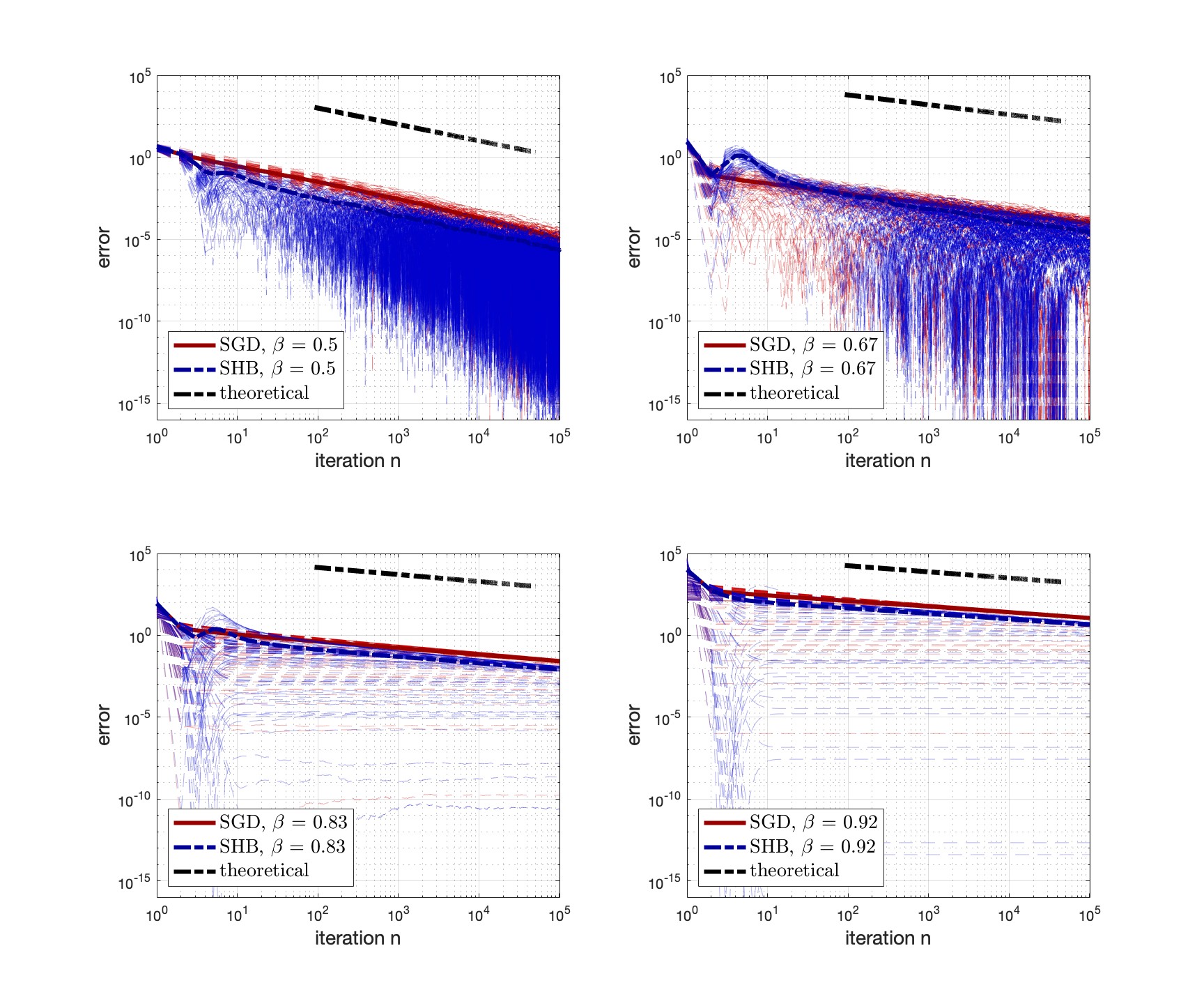}
\end{center}
\caption{Pathwise error $(f_p(X_n))_{n=1,\dots,N}$ of SGD and SHB for various choices of $\beta\in\{0.5, 0.67, 0.83, 0.92\}$. For each setting we have simulated $100$ runs of length , $N=10^5$. The bold lines correspond to the average error of SGD (red) and SHB (blue), and the black dash-dotted line corresponds to the theoretical rate $n^{-\frac{1}{4\beta-1}}$.}\label{fig:1}
\end{figure}

\paragraph{Details of the implementation:} Both algorithms have been implemented by hand using \texttt{MATLAB}. We have initialized both SGD and SHB with the initial state $X_1 \sim \frac12\cU([1.5,2.5])+\frac12\cU([-2.5,1.5])$ to force initials which are not close to the actual minimum $x^\ast = 0$. The initial step sizes $\gamma_1(\beta)$ for both algorithms are chosen as
\[\gamma_1(0.5)=0.2,\ \gamma_1(0.67) = 0.13,\ \gamma_1(0.83) = 0.004,\ \gamma_1(0.92) = 10^{-6} \]
through which we counteract the decreasing smoothness for $\beta\to1$. The momentum parameter for SHB is fixed for all $\beta$ as $\nu = 0.5$. The exact gradients $\nabla f_p$ are perturbed by independent additive noise following a standard normal distribution $\cN(0,1)$.

\section{Proof of Lemma~\ref{lem:as_rate_extension_beta_allgemein}}\label{app:proof-lem-as-rate-beta-allgemein}
In the following section, we present the proof of our super-martingale result.
\begin{proof}[Proof of Lemma~\ref{lem:as_rate_extension_beta_allgemein}]
In the following, we treat both cases $\beta = \frac12$ and $\beta \in (\frac12,1]$ separately. \\
\underline{$\beta = \frac{1}{2}$:} In this case, the inequality reduces to 
\[ \E[Y_{n+1}\mid \mathcal F_n] \le (1+c_1\gamma_n^2- c_2\gamma_n)Y_n + c_3\gamma_n^2.\]
By the choice of $\gamma_n$, there exists some $N>0$ and $\tilde c_1 >0$ such that $c_2 \gg_n -c_1 \gamma_n^2 \geq \tilde c_1 \gg_n$ for all $n \geq N$.
Hence, for all $n \geq N$
\[ \E[Y_{n+1}\mid \mathcal F_n] \le (1- \tilde c_1 \gamma_n)Y_n + c_3\gamma_n^2\]
such that the claim follows by~\citet[Lem. 1]{LY2022}. 

\underline{$\beta \in (\frac{1}{2},1]$:} The proof uses the elementary inequality 
\begin{equation} \label{eq:elementary}
(n+1)^{1-\eta} \le n^{1-\eta} + (1-\eta)n^{-\eta},
\end{equation}
which was also applied and proved in~\citet[Lem. 1]{LY2022}.
The aim is to apply the Robbins-Siegmund implication, Lemma~\ref{cor:RS_corollary}, in order to derive the almost sure convergence rate.
Let $1 \leq q< 2$ be arbitrary for now. The key step of the proof is the following computation
\begin{equation}\label{eq:recursion-Y_n-with-q}
\begin{split}
\E[Y_{n+1}\mid\cF_n] &\le (1+c_1 \gg_n^2)Y_n - c_2\gamma_n Y_n^{2\beta} + c_3 \gamma_n^2\\
 &= (1+c_1 \gg_n^2) Y_n - c_2\gamma_n^q Y_n + c_2\gamma_n^q Y_n - c_2\gamma_n Y_n^{2\beta} + c_3\gamma_n^2\\
 &=(1+c_1 \gg_n^2-c_2 \gamma_n^q) Y_n + c_2 \gamma_n \left(\gamma_n^{q-1} Y_n - Y_n^{2\beta} \right)+ c_3\gamma_n^2.
 \end{split}
\end{equation}
Similar to the case $\beta = \frac 12$ there exists some $N>0$ and $\tilde c_1 >0$ such that $c_2 \gg_n^q -c_1 \gamma_n^2 \geq \tilde c_1 \gg_n^q$ for all $n \geq N$.
Hence, for all $n \geq N$ we obtain the iterative inequality of the form
\begin{equation}
\begin{split}
\E[Y_{n+1}\mid\cF_n] &\le (1-\tilde c_1 \gamma_n^q) Y_n + c_2 \gamma_n \left(\gamma_n^{q-1} Y_n - Y_n^{2\beta} \right)+ c_3\gamma_n^2.
 \end{split}
\end{equation}
The function $x\mapsto ax-bx^{2\beta}$ takes it maximum at $\bar x = \left(\frac{a}{2b\beta}\right)^{\frac{1}{2\beta-1}}$ such that
\begin{equation}\label{eq:function-trick}
\begin{split}
    \gamma_n (\gamma_n^{q-1} Y_n - Y_n^{2\beta})&\le \frac{\gamma_n^{q+ \frac{q-1}{2\beta-1}}}{(2\beta)^{\frac{1}{2\beta-1}}}- \frac{\gamma_n^{1+ \frac{(q-1)2\beta}{2\beta-1}}}{(2\beta)^{\frac{2\beta}{2\beta-1}}}\\
    &=\frac{1}{(2\beta)^{\frac{1}{2\beta-1}}} \gamma_n^{\frac{2q\beta-1}{2\beta-1}}- \frac{1}{(2\beta)^{\frac{2\beta}{2\beta-1}}}\gamma_n^{\frac{2q\beta-1}{2\beta-1}}\\
    &= (2\beta)^{-\frac{1}{2\beta-1}} (1-\frac{1}{2\beta}) \gamma_n^{\frac{2q\beta-1}{2\beta-1}}
\end{split}
\end{equation}
holds almost surely. We define $\tilde c_2 = c_2 (2\beta)^{-\frac{1}{2\beta-1}} (1-\frac{1}{2\beta}) \in (0,\infty)$ for $\beta \in (\frac12,1)$ and proceed with
\begin{align}\label{eq:recursion-Y_n-with-q-after-function-trick}
\E[Y_{n+1}\mid\cF_n] &\le (1-\tilde c_1 \gamma_n^q) Y_n +\tilde c_2 \gamma_n^\frac{2\beta q-1}{2 \beta-1} + c_3 \gamma_n^2.
\end{align}

Next, we apply the elementary inequality \cref{eq:elementary} and choose $q$ such that $\frac{1}{2}<\theta\leq\frac{1}{q}\leq 1$. Moreover by the choice of $\gamma_n$, there exists some $c_4 >0$ such that $\tilde c_1 \gamma_n^q \geq \frac{c_4}{t^{q\theta}}$ for all $n \geq N$. It follows that for all $n \geq N$ 
\begin{align*}
&\E[ (n+1)^{1-\eta} Y_{n+1}\mid \cF_n] \\
&\le (n+1)^{1-\eta} (1-\tilde c_1\gamma_n^q) Y_n + (n+1)^{1-\eta}  \tilde c_2 \gamma_n^\frac{2\beta q-1}{2 \beta-1}+(n+1)^{1-\eta} c_3 \gamma_n^2\\
&\le (n^{1-\eta} + (1-\eta) n^{-\eta}) (1-\frac{c_4}{n^{q\theta}})Y_n + (n+1)^{1-\eta}  \tilde c_2 \gamma_n^\frac{2\beta q-1}{2 \beta-1}+(n+1)^{1-\eta} c_3 \gamma_n^2\\
&=\left(1+\frac{1-\eta}{n}-\frac{c_4}{n^{q\theta}}-\frac{c_4(1-\eta)}{n^{q\theta + 1}}\right)n^{1-\eta} Y_n 
+ (n+1)^{1-\eta}  \tilde c_2 \gamma_n^\frac{2\beta q-1}{2 \beta-1}+(n+1)^{1-\eta} c_3 \gamma_n^2.
\end{align*}
We set $\tilde c_3 = \max\{\tilde c_2, c_3\}$ such that for all $n \geq\ N$ 
\begin{align*}
&\E[ (n+1)^{1-\eta} Y_{n+1}\mid \cF_n] \\
&\le \left(1+\frac{1-\eta}{n}-\frac{c_4}{n^{q\theta}}-\frac{c_4(1-\eta)}{n^{q\theta + 1}}\right)n^{1-\eta} Y_n 
+ \tilde c_3 (n+1)^{1-\eta} (\gamma_n^\frac{2\beta q-1}{2 \beta-1}+\gamma_n^2).
\end{align*}

Observe that $q\theta\leq1$ by condition $\theta\le\frac{1}{q}$. Hence, there exists $\tilde c_4>0$ and $\tilde N >N$for sufficiently large $\tilde N\ge N$ such that for all $n \geq \tilde N$ we have
\begin{equation}\label{eq:help-lem-21}
    \E[ (n+1)^{1-\eta} Y_{n+1}\mid \cF_n] 
    \le (1 - \tilde c_4 \frac{1}{n^{q\theta}}) n^{1-\eta} Y_n + c_3(n+1)^{1-\eta} (\gamma_n^\frac{2\beta q-1}{2 \beta-1}+\gamma_n^2)
\end{equation}
In order to apply Robbins-Siegmund, more precisely Lemma~\ref{cor:RS_corollary}, we are going to verify the following three sufficient conditions:
\begin{align}
    &\sum_{n=\tilde N}^\infty \frac{1}{n^{q\theta}} = \infty, \label{eq:cond-1-lem21}\\
    &\sum_{n=\tilde N}^\infty n^{1-\eta-2\gt} <\infty, \label{eq:cond-2-lem21}\\
    &\sum_{n=\tilde N}^\infty n^{1-\eta-\frac{\gt (2\beta q-1)}{2\beta-1}}<\infty.\label{eq:cond-3-lem21}
\end{align}
Then, $Y_n \in o\left(\frac{1}{n^{1-\eta}}\right)$ almost surely. 

The first condition \cref{eq:cond-1-lem21} is obviously satisfied, since we assume $\theta\leq\frac{1}{q}$. For the second condition \cref{eq:cond-2-lem21} we may choose $\theta> 1-\frac \eta 2$ such that $1-\eta-2\theta<-1$. The third condition \cref{eq:cond-3-lem21} gives $1-\eta-\frac{\gt (2\beta q-1)}{2\beta-1}<-1$ which leads to the condition $\theta>\frac{(2-\eta) (2\beta -1)}{2\beta q-1}$. Hence, all together we obtain the sufficient condition 
\[\theta \in \left(\max\left\{\frac{(2-\eta) (2\beta -1)}{2\beta q-1},1-\frac{\eta}{2}\right\},\frac{1}{q}\right].\]
In the following, we consider the two cases separately that correspond to the maximum being either $1-\frac\eta 2$ or $\frac{(2-\eta) (2\beta-1)}{2\beta q -1}$. The first case occurs precisely for $\frac{1}{q}\leq \frac{2\beta}{4\beta-1}$, the latter one for $\frac 1 q \geq \frac{2\beta}{4\beta-1}$.\\
 Firstly, let $\frac{1}{q}\leq \frac{2\beta}{4\beta-1}$. In this situation the sufficient condition on $\gt$ simplifies to 
\[\theta \in \left(1-\frac\eta 2,\frac{1}{q}\right].\]
The interval is non-empty for $\frac1q > \frac{2-\eta}{2}$, which requires $\eta\in (\frac{4\beta-2}{4\beta-1},1)$.\\
Secondly, let $\frac1q \geq \frac{2\beta}{4\beta-1}$. In this situation the sufficient condition on $\gt$ simplifies to  
\[\theta \in \left(\frac{(2-\eta) (2\beta -1)}{2\beta q-1},\frac{1}{q}\right],\]
the interval is non-empty for $\frac{1}{q} < 2\beta\eta-2\beta+2-\eta$. Hence, $\frac 1q\in (\frac{2\beta}{4\beta-1},{2\beta\eta-2\beta+2-\eta})$ which requires the condition $\eta\in (\frac{4\beta-2}{4\beta-1},1)$.

Either case yields sufficient conditions on $\gt$ and $\eta$ (depending on the auxiliary variable $q$) under which $Y_n \in o\left(\frac{1}{n^{1-\eta}}\right)$ holds almost surely. We will now utilize the free variable $q$ to prove the claim.

\begin{itemize}
    \item Let $\gt \in (\frac 12, \frac{2\beta}{4\beta-1})$: We set $q= \frac{4\beta-1}{2\beta}$ and use the first case. The assumption $\eta > 2-2\gt= \max\{2-2\gt,\frac{\gt+2\beta-2}{2\beta-1}\}$ implies $\gt\in \left(1-\frac\eta 2,\frac{1}{q}\right]$. (Note that $ \eta > \frac{4\beta-2}{4\beta-1}$ is automatically fulfilled by $2-2\gt > \frac{4\beta-2}{4\beta-1}$ for this choice of $\gt$.)
    \item Let $\gt \in [\frac{2\beta}{4\beta-1},1)$: By assumption we have $\eta > \frac{\gt+2\beta-2}{2\beta-1}= \max\{2-2\gt, \frac{\gt+2\beta-2}{2\beta-1}\}$. We choose some $\frac 1q\in (\gt,{2\beta\eta-2\beta+2-\eta})$ and use the second case. (Note that $ \eta > \frac{4\beta-2}{4\beta-1}$ again is automatically fulfilled by $ \frac{\gt+2\beta-2}{2\beta-1}>\frac{4\beta-2}{4\beta-1}$ for this choice of $\gt$.) 
\end{itemize}
All in all we have proved that $\gt\in (\frac 12,2)$ implies $Y_n \in o\left( \frac{1}{n^{1-\eta}}\right)$ almost surely for all $\eta \in(\max\{2-2\gt, \frac{\gt+2\beta-2}{2\beta-1}\},1)$.
\end{proof}

\section{Proofs of Section~\ref{sec:localSGD}}\label{app:proof-localSGD}
In the following section, we present the proofs of \Cref{sec:localSGD}.
\subsection{Proof of Theorem~\ref{thm:conv-SGD-local-PL-xast}}
Suppose that the assumptions of Theorem~\ref{thm:conv-SGD-local-PL-xast} hold throughout this section. In contrast to the global gradient domination analysis we may assume w.l.o.g.~the uniform second moment bounds, i.e.~$A=B=0$, instead of the more general \cref{eq:ABC-condition} condition. Choosing $A,B>0$ would imply the bounded variance assumption of the gradient estimator. Note therefore, that the first term $A (f(x)-f(x^\ast))$ and the second term $B\lVert \nabla f(x)\rVert^2$ are both locally bounded by the local Lipschitz assumptions on $f$ and $\nabla f$.

Note that every isolated local minimum $\{x^\ast\}$ is a special case of an isolated compact connected set of local minima.  In this case it holds that $\beta = \beta_{x^\ast}$. If $\cX^\ast$ contains more then one point, we can unify the gradient domination property in a neighbourhood of $\cX^\ast$ due to compactness. The set $\cX^\ast$ has to be connected to assure that all local minima are on the same level $l$.

Recall, the outline of the proof is structured as follows:
\begin{itemize}
    \item First, we unify the gradient domination property around the set of local minima $\cX^\ast$ and obtain a radius $r$ such that the unified gradient domination property is fulfilled in all open balls with radius $r$ around $x^\ast\in\cX^\ast$ (Lemma~\ref{lem:unify-gradientdom-xast}). 
    \item Based on this we construct sets $\cU$, $\cU_1 \subseteq \R^d$ and the events $\Omega_n \in \Omega$ (see \cref{eq:h1}, \cref{eq:h2} and \cref{eq:h3}), such that $\Omega_\cU = \bigcap_{n} \Omega_n$ occurs with high probability. To be precise, $\cU_1$  and $\cU$ are neighborhoods of $\cX^\ast$ constructed such that the gradient domination property holds within this region, and when starting in $\cU_1$ the gradient trajectory does remain in $\cU$ for all gradient steps with high probability. Then, $\Omega_n$ describes the event that $X_k \in \cU$ for all $k\leq n$. 
    \item All following Lemmata before the proof of \cref{thm:conv-SGD-local-PL-xast} are devoted to show that $\bbP(\Omega_n) \geq 1-\delta$ for all $n \in\N$. This then proves Claim (i) of the Theorem. Claim (ii) and (iii) will be shown directly in the proof of \cref{thm:conv-SGD-local-PL-xast} at the end of this subsection. 
    \item In order to show $\bbP(\Omega_n) \geq 1-\delta$ we construct set $C_n$ and $E_n$ defined in \cref{eq:h4} and \cref{def:E_n-x-ast} such that $E_n \cap C_n \subset \Omega_{n+1}$ (Lemma~\ref{lem:tilde-r-n-xast}) while Lemma~\ref{lem:aux-lemma-x_n-close-to-xast} is used to prove this claim.
    \item The sets $E_n$ are such that $f(X_n)$ remains close to $f^\ast$. We exploit the unified gradient domination property to construct the sets $E_n$ (Lemma~\ref{lem:bar_D_iteration_x_ast}) and derive a recursive inequality in Lemma~\ref{lem:tilde-r-n-xast} c) to prove that this event occurs with high probability (Lemma~\ref{lem:prob_E_n_lager_1-delta-xast}). 
    \item The sets $C_n$ are such that $X_{n+1}$ remains close to $X_n$ and we exploit the finite variance assumption to show that these events occur with high probability (Lemma~\ref{lem:prob_hat-E_n_lager_1-delta}). 
\end{itemize}

We denote by 
$$\tilde\cB_{r}(x) = \{ y\in\R^d\, :\, ||x-y|| < r\} $$ the open ball with radius $r>0$ around $x\in\R^d$ and by 
$$\cB_{r}(x) = \{ y\in\R^d\, :\, ||x-y|| \leq r\} $$ the closed ball with radius $r>0$ around $x\in\R^d$. 

In the following Lemma we unify the gradient domination property around the set of local minima $\cX^\ast\subset \R^d$.
\begin{lemma}\label{lem:unify-gradientdom-xast}.
    There exists $r>0$, $\beta \in [\frac{1}{2},1]$ and $c>0$, such that for all $x \in \bigcup_{x^\ast\in \cX^\ast} \tilde\cB_{r}(x^\ast)$ it holds that
    \[ f(x) > l\ \textrm{ for }\ x\notin \cX^\ast \quad \textrm{ and } \quad ||\nabla f(x) || \geq c (f(x)-l)^\beta\,.\]
\end{lemma}
\begin{proof}
    By the local gradient domination property, for every $x^\ast \in \cX^\ast$ there exist $r_{x^\ast}>0$, $\beta_{x^\ast}\in[\frac 12,1]$ and $c_{x^\ast}>0$ such that 
    \[ ||\nabla f(x) || \geq c_{x^\ast} |f(x)-l|^{\beta_{x^\ast}}, \quad \forall x \in \cB_{r_{x^\ast}}(x^\ast).\]
    Moreover, w.l.o.g we can assume that $f(x) > l$ for all $ x \in \cB_{r_{x^\ast}}(x^\ast) \setminus\cX^\ast$, as $\cX^\ast$ is an isolated compact connected set of local minima (otherwise choose $r_{x^\ast}$ small enough).

    By the compactness of $\cX^\ast$ we can find a finite subset $\cY^\ast\subset\cX^\ast$, such that
    $$\tilde \cU := \bigcup_{y^\ast\in \cY^\ast} \tilde\cB_{r_{y^\ast}}(y^\ast) \supset \cX^\ast.$$ 
    Then, we define $\beta = \max_{y^\ast\in \cY^\ast} \beta_{y^\ast}$ and $c= \min_{y^\ast\in \cY^\ast} c_{y^\ast}$. For any $x \in \tilde \cU$ there exits $y^\ast \in \cY^\ast$ such that 
    $$ ||\nabla f(x)|| \geq c_{y^\ast} (f(x)-l)^{\beta_{y^\ast}} \geq c (f(x)-l)^{\beta}.$$
    Thus, there exists an open neighbourhood $\tilde \cU$ of $\cX^\ast$ and $\beta \in [\frac{1}{2},1]$, $c>0$, such that for all $x \in \tilde \cU$ it holds that
    \[ f(x) > l  \textrm{ for } x\notin \cX^\ast  \quad \textrm{ and } \quad ||\nabla f(x) || \geq c (f(x)-l)^\beta.\]
    As $\tilde \cU$ is open by definition and $\cX^\ast \subset \tilde \cU$, we can find a radius $r>0$, such that 
    $\bigcup_{x^\ast\in \cX^\ast} \tilde\cB_{r}(x^\ast) \subseteq \tilde\cU.$ This proves the claim.
\end{proof}

\begin{remark}
    It is noteworthy that the unified gradient domination property obtained in the previous Lemma does not require an absolute value, as $f(x) \geq l$ for all  $x \in \bigcup_{x^\ast\in \cX^\ast} \tilde\cB_{r}(x^\ast)$. 
    This is crucial to obtain the recursive inequalities in Lemma~\ref{lem:bar_D_iteration_x_ast} and we will exploit this also in the proof of \cref{thm:conv-SGD-local-PL-xast} to obtain the convergence rates. 
\end{remark}

In the following let $\mathbf{r}>0$, $c>0$ and $\beta \in [\frac{1}{2},1]$ chosen as in the previous Lemma, such that the unified gradient domination property holds for all $x \in \bigcup_{x^\ast\in \cX^\ast} \tilde\cB_{\mathbf{r}}(x^\ast)$.
Further define 
    $$s = \inf\left\{ f(x)-l \,: \, x \in \bigcup_{x^\ast\in \cX^\ast} \cB_{\frac{3\mathbf{r}}{4}}(x^\ast) \setminus \bigcup_{x^\ast\in \cX^\ast} \tilde\cB_{\frac{\mathbf{r}}{2}}(x^\ast) \right\}.$$
\begin{lemma}
    It holds that $s>0$.
\end{lemma}
\begin{proof}
    If $s=0$, then there exists a sequence $(x_n) \in \bigcup_{x^\ast\in \cX^\ast} \cB_{\frac{3\mathbf{r}}{4}}(x^\ast) \setminus \bigcup_{x^\ast\in \cX^\ast} \tilde\cB_{\frac{\mathbf{r}}{2}}(x^\ast)$ with $f(x_n) \to l$ for $n \to \infty$. 
    By definition of the set and compactness (boundedness) of $\cX^\ast$, the sequence $x_n$ is bounded: 
    $$ ||x_n|| \leq \frac{3\mathbf{r}}{4} + \sup_{x^\ast\in\cX^\ast} ||x^\ast||< \infty.$$
    Hence, there is a convergent sub-sequence $(x_{n_k})$ with $x_{n_k} \to x$ for $k \to \infty$ and by continuity of $f$ it holds that $f(x) = l$.    
    Further, it holds for all $x^\ast \in\cX^\ast$ that $||x_n -x^\ast || \geq \frac{\mathbf{r}}{2}$ for all $n \in \N$ such that $\inf_{x^\ast\in\cX^\ast}||x - x^\ast|| \geq \frac{\mathbf{r}}{2}$.

    On the other hand, by construction we have that $x \in \overline{\bigcup_{x^\ast\in \cX^\ast} \cB_{\frac{3\mathbf{r}}{4}}(x^\ast) \setminus \bigcup_{x^\ast\in \cX^\ast} \tilde\cB{\frac{\mathbf{r}}{2}}(x^\ast)} \subset \overline{\bigcup_{x^\ast\in \cX^\ast} \cB_{\frac{3\mathbf{r}}{4}}(x^\ast)} \subset\bigcup_{x^\ast\in \cX^\ast} \tilde\cB_{\mathbf{r}}(x^\ast)$. And as $f(y) > l$ for all $y \in \tilde\cB_{\mathbf{r}}(x^\ast)\setminus \cX^\ast$ we deduce from $f(x) =l$ that $x \in \cX^\ast$. This is a contradiction to $\inf_{x^\ast\in\cX^\ast}||x - x^\ast|| \geq \frac{\mathbf{r}}{2}$.
\end{proof}

We choose $\epsilon>0$, such that $2\epsilon+\sqrt{\epsilon}<s$.
We define the sets
\begin{align}
    \cU_1 &= \{ x\in\R^d\,: \, \inf_{x^\ast\in\cX^\ast}||x-x^\ast|| < \frac{\mathbf{r}}{2}, f(x)-l \leq \frac{\epsilon}{2}\}\label{eq:h1}\\
    \cU &= \{  x\in\R^d\,: \, \inf_{x^\ast\in\cX^\ast}||x-x^\ast|| < \frac{\mathbf{r}}{2} \}\label{eq:h2}
\end{align}
which are subsets of $\R^d$ and the decreasing sequence of events 
\begin{align}
    \Omega_n &= \{X_k \in \cU \textrm{ for all } k \leq n\}\label{eq:h3} \\
    C_n &= \{  ||X_{k+1} - X_{k}||\leq \frac{\mathbf{r}}{4} \textrm{ for all } k \leq n\},\label{eq:h4}
\end{align}
and $C_0 = \Omega$, which are measurable sets in $(\Omega, \cF,\bbP)$.

In order to prove Theorem~\ref{thm:conv-SGD-local-PL-xast} we will show that $\Omega_n$ has probability at least $1-\delta$ for all $n\in\N$. To do this, we construct another sequence of events $(\hat E_n)$ with $\hat E_n \subset \Omega_n$ which occur with probability at least $1-\delta$ for any $n \in \N$. 

Therefore, we fix the notation $D_n := f(X_n)-l$, $\xi_{n+1} := -\langle \nabla f(X_n),Z(X_n,\zeta_{n+1})\rangle$ and $\mathbf{1}_\cA$ denotes the indicator function for a measurable set $\cA$ in $(\Omega,\cF,\bbP)$, i.e. $\mathbf{1}_\cA(\omega) = 1$ if $\omega \in \cA$ and $\mathbf{1}_\cA(\omega) = 0$ if $\omega \notin \cA$. We prove the following (recursive) inequalities. 

\begin{lemma}\label{lem:bar_D_iteration_x_ast}
    If $\beta = \frac 12$, then it holds that
    \begin{equation}\label{eq:lemE1-eq3-xast}
        \begin{split}
         {D}_{n+1} \mathbf{1}_{\Omega_n}&\le (1-\gg_n c^2) D_n\mathbf{1}_{\Omega_n}  + \gg_n \xi_{n+1}\mathbf{1}_{\Omega_n} + \frac{L \gg_n^2}{2} \mathbf{1}_{\Omega_n} \lVert V_{n+1}(X_n) \rVert^2, \\
        &\le {D}_1 \prod_{k=1}^{n}(1- \gg_k c^2) \mathbf{1}_{\Omega_n}+ \sum_{k=1}^{n} \left(\prod_{j=k}^{n}(1- \gg_j c^2)\right) \gg_k \xi_{k+1} \mathbf{1}_{\Omega_n}\\ &\quad + \frac{L}{2}\sum_{k=1}^{n}  \gg_k^2 \lVert V_{k+1}(X_k) \rVert^2 \mathbf{1}_{\Omega_n}.
        \end{split}
    \end{equation}
    If $\beta \in (\frac{1}{2},1]$, for any $1 \leq q< 2$, it holds that
    \begin{equation}\label{eq:lemE1-eq1-xast}
        \begin{split}
         {D}_{n+1} \mathbf{1}_{\Omega_n} &\le (1- \gg_n^q c^2)  {D}_n \mathbf{1}_{\Omega_n}+  (2\beta)^{-\frac{1}{2\beta-1}} (1-\frac{1}{2\beta}) {c^2}\gg_n^\frac{2\beta q-1}{2\beta-1}+ \gg_n \xi_{n+1} \mathbf{1}_{\Omega_n} + \frac{L \gg_n^2}{2} \lVert V_{n+1}(X_n) \rVert^2 \mathbf{1}_{\Omega_n} \\
        &\le {D}_1 \prod_{k=1}^{n}(1- \gg_k^q c^2) +  \tilde c \sum_{k=1}^{n} \gg_k^\frac{2\beta q-1}{2\beta-1} + \sum_{k=1}^{n} \left(\prod_{j=k}^{n}(1- \gg_j^q c^2)\right)\gg_k \xi_{k+1} \mathbf{1}_{\Omega_n}\\ &\quad+ \frac{L}{2}\sum_{k=1}^{n} \gg_k^2 \lVert V_{k+1}(X_k) \rVert^2 \mathbf{1}_{\Omega_n} ,
        \end{split}
    \end{equation}
    for $\tilde c = (2\beta)^{-\frac{1}{2\beta-1}} (1-\frac{1}{2\beta}) c^2$.
\end{lemma}

\begin{proof}
From $L$-smoothness we can deduce that 
\begin{align*}
    D_{n+1} &\le D_n - \gg_n \langle \nabla f(X_n), V_{n+1}(X_n) \rangle + \frac{L \gg_n^2}{2} \lVert V_{n+1}(X_n) \rVert^2 \\
    &=  D_n - \gg_n \lVert \nabla f(X_n) \rVert^2 - \gg_n \langle \nabla f(X_n), Z(X_n, \zeta_{n+1}) \rangle + \frac{L \gg_n^2}{2} \lVert V_{n+1}(X_n) \rVert^2 \\
    &= D_n - \gg_n \lVert \nabla f(X_n) \rVert^2 + \gg_n \xi_{n+1} + \frac{L \gg_n^2}{2} \lVert V_{n+1}(X_n) \rVert^2
\end{align*}
for $Z(X_n, \zeta_{n+1})$ from 
Assumption~\ref{ass:ABC} and $\xi_{n+1} = -\langle \nabla f(X_n), Z(X_n,\zeta_{n+1}) \rangle$.

We separate the two cases of $\beta$:

\underline{$\beta = \frac 12$:}
Iterating this inequality and using $ \mathbf{1}_{\Omega_{n+1}} \leq \mathbf{1}_{\Omega_n}$ it follows that
\begin{equation}\label{eq:bar_D_iteration_x_ast-with-beta12}
\begin{split}
    D_{n+1} \mathbf{1}_{\Omega_n} 
    &\leq D_n \mathbf{1}_{\Omega_n} - \gg_n \mathbf{1}_{\Omega_n} \lVert \nabla f(X_n) \rVert^2 + \gg_n \mathbf{1}_{\Omega_n} \xi_{n+1} + \frac{L \gg_n^2}{2} \mathbf{1}_{\Omega_n} \lVert V_{n+1}(X_n) \rVert^2 \\
    &\le D_n \mathbf{1}_{\Omega_n} - \gg_n c^2 (f(X_n)-l) \mathbf{1}_{\Omega_n} + \gg_n \mathbf{1}_{\Omega_n}  \xi_{n+1} + \frac{L \gg_n^2}{2}\mathbf{1}_{\Omega_n} \lVert V_{n+1}(X_n) \rVert^2 \\
    &= (1-\gg_n c^2) D_n  \mathbf{1}_{\Omega_n} + \gg_n \xi_{n+1}\mathbf{1}_{\Omega_n} + \frac{L \gg_n^2}{2} \mathbf{1}_{\Omega_n} \lVert V_{n+1}(X_n) \rVert^2, \\
    &\le  {D}_1 \prod_{k=1}^{n}(1- \gg_k c^2) + \sum_{k=1}^{n}\left(\prod_{j=k}^{n}(1- \gg_j c^2)\right) \gg_k \xi_{k+1} \mathbf{1}_{\Omega_n} \\ 
    &\quad+ \frac{L}{2}\sum_{k=1}^{n}\left(\prod_{j=k}^{n}(1- \gg_j c^2)\right) \gg_k^2 \lVert V_{k+1}(X_k) \rVert^2\mathbf{1}_{\Omega_n}\\
    &\le  {D}_1 \prod_{k=1}^{n}(1- \gg_k c^2) + \sum_{k=1}^{n}\left(\prod_{j=k}^{n}(1- \gg_j c^2)\right) \gg_k \xi_{k+1} \mathbf{1}_{\Omega_n} \\ 
    &\quad+ \frac{L}{2}\sum_{k=1}^{n} \gg_k^2 \lVert V_{k+1}(X_k) \rVert^2\mathbf{1}_{\Omega_n},
\end{split}    
\end{equation}
where we used that the unified gradient domination property holds for all $X_k$, $k\leq n$ on the event $\Omega_n$. 

\underline{$\beta \in (\frac 12,1]$:}
Similarly, the unified gradient domination property yields the claimed inequality for any $1 \leq q < 2$:
\begin{equation}\label{eq:bar_D_iteration_x_ast}
\begin{split}
    D_{n+1} \mathbf{1}_{\Omega_n} 
    &\leq D_n \mathbf{1}_{\Omega_n} - \gg_n \mathbf{1}_{\Omega_n} \lVert \nabla f(X_n) \rVert^2 + \gg_n \mathbf{1}_{\Omega_n} \xi_{n+1} + \frac{L \gg_n^2}{2} \mathbf{1}_{\Omega_n} \lVert V_{n+1}(X_n) \rVert^2 \\
    &\le D_n \mathbf{1}_{\Omega_n} - \gg_n c^2 (f(X_n)-l)^{2\beta} \mathbf{1}_{\Omega_n} + \gg_n \mathbf{1}_{\Omega_n}  \xi_{n+1} + \frac{L \gg_n^2}{2}\mathbf{1}_{\Omega_n} \lVert V_{n+1}(X_n) \rVert^2 \\
    &= D_n \mathbf{1}_{\Omega_n}  - \gg_n c^2  D_n^{2\beta} \mathbf{1}_{\Omega_n}  + \gg_n \xi_{n+1}\mathbf{1}_{\Omega_n} + \frac{L \gg_n^2}{2} \mathbf{1}_{\Omega_n} \lVert V_{n+1}(X_n) \rVert^2, \\
    &= (1- \gg_n^q c^2)  {D}_n \mathbf{1}_{\Omega_n} + \gg_n c^2 (\gg_n^{1-q} {D}_n   -{D}_n^{2\beta})\mathbf{1}_{\Omega_n}  + \gg_n \xi_{n+1} \mathbf{1}_{\Omega_n}+ \frac{L \gg_n^2}{2} \lVert V_{n+1}(X_n) \rVert^2 \mathbf{1}_{\Omega_n}\\
    &\le (1- \gg_n^q c^2)  {D}_n \mathbf{1}_{\Omega_n}  +  (2\beta)^{-\frac{1}{2\beta-1}} (1-\frac{1}{2\beta}) {c^2}\gg_n^\frac{2\beta q-1}{2\beta-1}+ \gg_n \xi_{n+1} \mathbf{1}_{\Omega_n} + \frac{L \gg_n^2}{2} \lVert V_{n+1}(X_n) \rVert^2\mathbf{1}_{\Omega_n} \\
    &\le {D}_1 \prod_{k=1}^{n}(1- \gg_k^q c^2) \mathbf{1}_{\Omega_n} +  \tilde c \sum_{k=1}^{n} \left(\prod_{j=k}^{n}(1- \gg_j^q c^2)\right)\gg_k^\frac{2\beta q-1}{2\beta-1}\\ &\quad + \sum_{k=1}^{n}\left(\prod_{j=k}^{n}(1- \gg_j^q c^2)\right) \gg_k \xi_{k+1} \mathbf{1}_{\Omega_n}+ \frac{L}{2}\sum_{k=1}^{n}  \left(\prod_{j=k}^{n}(1- \gg_j^q c^2)\right)\gg_k^2 \lVert V_{k+1}(X_k) \rVert^2\mathbf{1}_{\Omega_n}\\
    &\le {D}_1 \prod_{k=1}^{n}(1- \gg_k^q c^2) \mathbf{1}_{\Omega_n} +  \tilde c \sum_{k=1}^{n} \gg_k^\frac{2\beta q-1}{2\beta-1 }+ \sum_{k=1}^{n}\left(\prod_{j=k}^{n}(1- \gg_j^q c^2)\right) \gg_k \xi_{k+1} \mathbf{1}_{\Omega_n}\\ &\quad + \frac{L}{2}\sum_{k=1}^{n} \gg_k^2 \lVert V_{k+1}(X_k) \rVert^2\mathbf{1}_{\Omega_n},
\end{split}    
\end{equation}
for $\tilde c = (2\beta)^{-\frac{1}{2\beta-1}} (1-\frac{1}{2\beta}) c^2$ from the function trick \cref{eq:function-trick} which we applied in the forth inequality. We also used that the unified gradient domination property holds for all $X_k$, $k\leq n$ on the event $\Omega_n$ .
\end{proof}

For $\beta \in(\frac 12,1]$ we know from the proof of Lemma~\ref{lem:as_rate_extension_beta_allgemein} that we can choose the auxiliary parameter $q$ from the previous lemma in such a way, that $\sum_{n=1}^{\infty} n^{1-\eta} \gg_n^\frac{2\beta q-1}{2\beta-1}$ is convergent for all $\eta \in (\max\{2-2\gt,\frac{\gt+2\beta-2}{2\beta-1}\},1)$ (Condition (iii) to apply Lemma~\ref{cor:RS_corollary}). As
 $\eta<1$, it follows that $\sum_{n=1}^{\infty} \gg_n^\frac{2\beta q-1}{2\beta-1} <\infty$ holds true for all these choices of $q$. Now define 
\begin{align*}
&M_n = \sum_{k=1}^{n}\left(\prod_{j=k}^{n}(1- \gg_j c^2)\right) \gg_k \xi_{k+1}  \mathbf{1}_{\Omega_k}, \quad M_n^{(q)} = \sum_{k=1}^{n}\left(\prod_{j=k}^{n}(1- \gg_j^q c^2)\right) \gg_k \xi_{k+1}  \mathbf{1}_{\Omega_k} \\
&\textrm{ and } \quad S_n = \frac{L}{2}\sum_{k=1}^{n}  \gg_k^2 \lVert V_{k+1}(X_k) \rVert^2\mathbf{1}_{\Omega_k}.
\end{align*}
Then, $(M_n)$ and $(M_n^{(q)})$ are $(\cF_{n+1})$-martingales with zero mean and $(S_n)$ is a $(\cF_{n+1})$-sub-martingale by Assumption~\ref{ass:ABC}. Note that by the choice of $\gg_n$ we have that $\sum_n \gg_n^2 <\infty$ and hence $\E[S_n] < \infty$ for all $n \in\N$. 

Next, define $R_n = M_n^2 + S_n$ and $R_n = (M_n^{(q)})^2 + S_n$ respectively (with some abuse of notation), for every $n\in\N$. Moreover, let 
\begin{equation}\label{def:E_n-x-ast}
    E_n = \{ R_k < \epsilon \textrm{ for all } k\leq n \}.
\end{equation}
which is an $\cF_{n+1}$-measurable event on $(\Omega, \cF,\bbP)$.
We define $R_0 = 0$ such that $E_0 = \Omega$.

Now let $\hat E_n = E_n \cap C_n$, then we will first show, that $\hat E_n$ fulfills the property $\hat E_n \subset \Omega_{n+1}$ for all $n\in\N$ in Lemma~\ref{lem:tilde-r-n-xast} and then that $\hat E_n$ occurs with probability at least $1-\delta$ in Lemma~\ref{lem:prob_hat-E_n_lager_1-delta}.

To prove that $\hat E_n \subset \Omega_{n+1}$ we need one more auxiliary result.
\begin{lemma}\label{lem:aux-lemma-x_n-close-to-xast}
    Suppose $x,y \in \R^d $ such that 
    \begin{enumerate}
        \item $\inf_{x^\ast\in\cX^\ast} ||x-x^\ast|| <\frac{\mathbf{r}}{2}$,
        \item $f(y) - l < s$,
        \item $||x-y||\leq \frac{\mathbf{r}}{4}$.
    \end{enumerate}
    Then it holds that $\inf_{x^\ast\in\cX^\ast} ||y-x^\ast||<\frac{\mathbf{r}}{2}$.
\end{lemma}
\begin{proof}
    By triangle inequality we have that $\inf_{x^\ast\in\cX^\ast} ||y-x^\ast||\leq \frac{3\mathbf{r}}{4}$, i.e there exists $x^\ast \in \cX^\ast$ such that $||y-x^\ast||\leq \frac{3\mathbf{r}}{4}$. Suppose now, that $\inf_{x^\ast\in\cX^\ast} ||y-x^\ast||\geq \frac{\mathbf{r}}{2}$, this means that $y \in \bigcup_{x^\ast\in\cX^\ast} \cB_{\frac{3r}{4}}(x^\ast) \setminus \bigcup_{x^\ast\in\cX^\ast} \tilde\cB_{\frac{\mathbf{r}}{2}}(x^\ast)$. By the definition of \[s = \inf \left \{ f(z)-l \,:\; z \in \bigcup_{x^\ast\in\cX^\ast} \cB_{\frac{3\mathbf{r}}{4}}(x^\ast) \setminus \bigcup_{x^\ast\in\cX^\ast} \tilde\cB_{\frac{\mathbf{r}}{2}}(x^\ast)\right\}\] this contradicts the second assumption $f(y) - l < s$.
\end{proof}

We deduce the following relations on the constructed sets:
\begin{lemma}\label{lem:tilde-r-n-xast}
    For $\beta \in (\frac{1}{2},1]$ let $\gg_n \leq \gg_1$ be sufficiently small such that $\sum_{n=1}^{\infty} \gg_n^\frac{2\beta q-1}{2\beta-1} < \frac{\epsilon}{2 \tilde c}$, and for $\beta=\frac{1}{2}$ let $\gg_1>0$ be arbitrary. Furthermore, assume that the initial $X_1 \in \cU_1$ almost surely.  Then, 
    \begin{enumerate}
        \item[a)] $E_{n+1} \subset E_n$, $\hat E_{n+1} \subset \hat E_n$ and $\Omega_{n+1} \subset \Omega_n$
        \item[b)] $\hat E_n \subset \Omega_{n+1}$
        \item[c)] Define the events $\tilde{E}_n = E_{n-1} \setminus E_n = E_{n-1} \cup \{R_n\geq\epsilon\}$. Then, for $\tilde{R}_n = R_n \mathbf{1}_{E_{n-1}}$, there exists a $\tilde C>0$ such that
        \begin{equation*}
            \E[\tilde{R}_{n}] \le \E[\tilde{R}_{n-1}] + \gamma_n^2 [G^2 C^2 + G^2 + C] - \epsilon \bbP(\tilde{E}_{n-1}).
        \end{equation*}
    \end{enumerate}
\end{lemma}

\begin{proof}
   a) Follows by definition of the events. 
   
   b) Note that $\hat E_0 = \Omega = \Omega_1$ because $$X_1 \in \cU_1 = \{x: \inf_{x^\ast\in\cX^\ast}|| x -x^\ast || < \frac{\mathbf{r}}{2}, f(x) - l \leq \frac{\epsilon}{2}\} \subset \{x: \inf_{x^\ast\in\cX^\ast}|| x -x^\ast || < \frac{\mathbf{r}}{2}\} = \Omega_1$$ almost surely.
    We prove the assertion by induction. Let $\omega \in\hat E_n$. Since $\hat E_n \subset \hat E_{n-1} \subset \Omega_n$ by induction assumption, we have $\omega \in \Omega_n$ and thus $\omega\in \Omega_k$ for all $k \le n$. 
    We will apply Lemma~\ref{lem:aux-lemma-x_n-close-to-xast} with $x = X_n(\omega)$ and $y = X_{n+1}(\omega)$. By definition it holds that $\omega\in \hat E_n$ implies condition $3.$ and $\omega \in \Omega_{n}$ implies condition $1.$ of Lemma~\ref{lem:aux-lemma-x_n-close-to-xast}. It remains to show condition $2.$, then it follows that $\inf_{x^\ast\in\cX^\ast} ||X_{n+1}(\omega)-x^\ast||< \frac{\mathbf{r}}{2}$, i.e. $X_{n+1}(\omega) \in \cU$ and by $\omega \in \Omega_n$ we deduce $\omega \in \Omega_{n+1}$.
    
    To Prove condition $2.$ we separate both cases for $\beta$:
    
    \underline{$\beta =\frac{1}{2}$:}
    The inequality \cref{eq:lemE1-eq3-xast} and the induction hypothesis yield
    \begin{align*}
        D_{n+1}(\omega) &= D_{n+1}(\omega) \mathbf{1}_{\Omega_n}(\omega)\\
        &\le D_1(\omega) \prod_{k=1}^{n}(1- \gg_k c^2) + \sum_{k=1}^{n} \left(\prod_{j=k}^{n}(1- \gg_j c^2)\right)\gg_k \xi_{k+1}(\omega) \mathbf{1}_{\Omega_n}(\omega)\\ &\quad+ \frac{L}{2}\sum_{k=1}^{n} \gg_k^2 \lVert V_{k+1}(X_k(\omega)) \rVert^2 \mathbf{1}_{\Omega_n}(\omega) \\
        &= D_1(\omega) \prod_{k=1}^{n}(1- \gg_k c^2) + \sum_{k=1}^{n} \left(\prod_{j=k}^{n}(1- \gg_j c^2)\right)\gg_k \xi_{k+1}(\omega) \mathbf{1}_{\Omega_k}(\omega)\\ &\quad+ \frac{L}{2}\sum_{k=1}^{n} \gg_k^2 \lVert V_{k+1}(X_k(\omega)) \rVert^2 \mathbf{1}_{\Omega_k}(\omega) \\
        &\le \frac{\epsilon}{2} +\sqrt{R_n(\omega)} + R_n(\omega) \\
        &\le 2 \epsilon + \sqrt{\epsilon} < s,
    \end{align*}
    where the equation in the third line is due to $\omega \in \Omega_k$ for all $k\leq n$ by induction.
    
    \underline{$\beta \in (\frac{1}{2},1]$:}  Similarly, we obtain from \cref{eq:lemE1-eq1}
    \begin{align*}
        D_{n+1}(\omega) &= D_{n+1}(\omega)\mathbf{1}_{\Omega_n}(\omega)\\
        &\le D_1(\omega) \mathbf{1}_{\Omega_n}(\omega) \prod_{k=1}^{n}(1- \gg_k^q c^2) +\tilde c \sum_{k=1}^{n} \gg_k^\frac{2\beta q-1}{2\beta-1}+ \sum_{k=1}^{n} \left(\prod_{j=k}^{n}(1- \gg_j^q c^2)\right)\gg_k \xi_{k+1}(\omega) \mathbf{1}_{\Omega_n}(\omega)\\&\quad + \frac{L}{2}\sum_{k=1}^{n}  \gg_k^2 \lVert V_{k+1}(X_k(\omega)) \rVert^2 \mathbf{1}_{\Omega_n}(\omega) \\
        &=D_1(\omega) \prod_{k=1}^{n}(1- \gg_k^q c) + \sum_{k=1}^{n} \left(\prod_{j=k}^{n}(1- \gg_j c^2)\right)\gg_k \xi_{k+1}(\omega) \mathbf{1}_{\Omega_k}(\omega)\\ &\quad+ \frac{L}{2}\sum_{k=1}^{n}  \gg_k^2 \lVert V_{k+1}(X_k(\omega)) \rVert^2 \mathbf{1}_{\Omega_k}(\omega) \\
        &\le \frac{\epsilon}{2}+\frac{\epsilon}{2}+ \sqrt{R_n(\omega)} + R_n(\omega) \\
        &\le 2 \epsilon + \sqrt{\epsilon}< s.
    \end{align*}
    We used in both cases that that $\prod_{k=1}^{n}(1- \gg_k^{q} c)\leq 1$ and the choice of $\epsilon$ such that $2 \epsilon + \sqrt{\epsilon}< s$.
    This proves that condition $2.$ in Lemma~\ref{lem:aux-lemma-x_n-close-to-xast} is also satisfied which concludes the induction.
    
    c) Without loss of generality we consider the case $\beta=1/2$. The computations for $\beta\in(1/2,1]$ follow in line by replacing $M_n$ with $M_n^{(q)}$. By definition it holds that $E_{n} = E_{n-1}\setminus (E_{n-1}\setminus E_{n}) = E_{n-1} \setminus \tilde{E}_{n}$. Then we have
    \begin{align*}
        \tilde{R}_{n} &= R_n \mathbf{1}_{E_{n-1}} \\
        &= R_{n-1} \mathbf{1}_{E_{n-1}} + (R_n - R_{n-1}) \mathbf{1}_{E_{n-1}} \\
        &= R_{n-1} \mathbf{1}_{E_{n-2}} - R_{n-1} \mathbf{1}_{\tilde{E}_{n-1}} + (R_n - R_{n-1}) \mathbf{1}_{E_{n-1}} \\
        &= \tilde{R}_{n-1} - R_{n-1} \mathbf{1}_{\tilde{E}_{n-1}} + (R_n - R_{n-1}) \mathbf{1}_{E_{n-1}}
    \end{align*}
    and for the last term
    \begin{align*}
        R_n -R_{n-1} &= M_n^2 - M_{n-1}^2 + S_n - S_{n-1} \\
        &= \gamma_n^2 (1-\gamma_n c^2)^2 \xi_{n+1}^2 \mathbf{1}_{\Omega_n}+ 2 \gamma_n (1-\gamma_n c) \xi_{n+1} \mathbf{1}_{\Omega_n} M_{n-1} + \gamma_n^2 \frac{L}{2} \lVert V_{n+1}(X_n)\rVert^2 \mathbf{1}_{\Omega_n}.
    \end{align*}
    We treat each of the summands on the RHS seperately.
    It follows from the $G$-Lipschitz continuity and bounded variance assumption in Theorem~\ref{thm:conv-SGD-local-PL-xast}, that
    \begin{align}
            \E[\xi_{n+1}^2 \mathbf{1}_{\Omega_n}] &= \E[\langle \nabla f(X_n), V_{n+1}(X_n)-\nabla f(X_n)\rangle^2 \mathbf{1}_{\Omega_n}]\notag\\ &\leq \E[\lVert \nabla f(X_n)\rVert^2 (\lVert V_{n+1}(X_{n})\rVert^2+1) \mathbf{1}_{\Omega_n} ] \leq G^2  (C^2+1), \notag \\
            \E[\xi_{n+1}(1-\gamma_n c) M_{n-1} \mathbf{1}_{\Omega_n}] &= \E[ \E[\xi_{n+1}|\cF_n] M_{n-1} \mathbf{1}_{\Omega_n}] =0, \notag \\
            \E[\lVert V_{n+1}(X_n)\rVert^2 \mathbf{1}_{\Omega_n}]  &\leq  C \label{eq:V_n^2-abschaetung}.
        \end{align} 
    For the term $R_{n-1} \mathbf{1}_{\tilde{E}_{n-1}}$ we have 
    \begin{equation*}
        \E[R_{n-1} \mathbf{1}_{\tilde{E}_{n-1}}] \ge \epsilon \bbP(\tilde{E}_{n-1}).
    \end{equation*}
    Using $(1-\gamma_n c)<1$ and putting all together we obtain the claim 
    \begin{equation*}
        \E[\tilde{R}_{n}] \le \E[\tilde{R}_{n-1}] + \gamma_n^2 [G^2 C^2 + G^2 + C] - \epsilon \bbP(\tilde{E}_{n-1}).
    \end{equation*}
\end{proof}

\begin{lemma}\label{lem:prob_E_n_lager_1-delta-xast}
    Let $\delta >0$ be a tolerance level and $\gg_n \leq \gg_1$ be sufficiently small such that $\sum_{n=1}^{\infty} \gg_n^2 < \frac{\delta \epsilon}{2(G^2 C^2 + G^2 + C)}$ and the condition in Lemma~\ref{lem:tilde-r-n-xast} is fulfilled. 
    Then, we have 
    \begin{align*}
        \bbP(E_n) \geq 1-\frac{\delta}{2}.
    \end{align*}
\end{lemma}
\begin{proof}
    The proof is along the lines of the proof of Proposition D2 in \cite{mertikopoulos2020sure}. For completeness we repeat the arguments. 
    First, observe that 
    \begin{equation*}
        \bbP(\tilde E_{n-1}) = \bbP(E_{n-1}\setminus E_n) = \bbP(E_{n-1} \cap \{R_n\geq \epsilon\})= \E[\mathbf{1}_{E_{n-1}} \mathbf{1}_{\{R_n>\epsilon\}}]\leq  \E[\mathbf{1}_{E_{n-1}} \frac{R_n}{\epsilon}]= \frac{\E[\tilde R_n]}{\epsilon}\,.
    \end{equation*}
    On the other hand it follows from Lemma~\ref{lem:tilde-r-n} that
    \begin{equation}
        \epsilon \bbP(\tilde E_n) \leq \E[\tilde R_n] \leq \E[\tilde R_0] + [G^2 C^2 + G^2 + C] \sum_{k=1}^n \gg_k^2 - \epsilon \sum_{k=0}^n \bbP(\tilde E_{k-1}) .
    \end{equation}
    Rearranging everything yields
    \begin{equation*}
        \sum_{k=0}^n \bbP(\tilde E_{k}) \leq \frac{[G^2 C^2 + G^2 + C] \Gamma}{\epsilon}
    \end{equation*}
    with $\Gamma = \sum_{n=1}^\infty \gg_n^2 $. By the assumption on the step size $\frac{[G^2 C^2 + G^2 + C] \Gamma}{\epsilon}<\frac \delta 2$ and moreover since the events $\tilde E_n$ are disjoint we obtain
    \begin{equation}
        \bbP(\bigcup_{k=0}^n \tilde E_k) = \sum_{k=0}^n \bbP(\tilde E_{k})\leq \frac \delta 2
    \end{equation}
    implying that 
    \begin{equation}
         \bbP(E_n) = \bbP(\bigcap_{k=0}^{n}\tilde E_k^c) \geq 1-\frac \delta 2\,.
    \end{equation}
\end{proof}

\begin{lemma}\label{lem:prob_hat-E_n_lager_1-delta}
    Let $\delta >0$ be a tolerance level and $\gg_n \leq \gg_1$ be sufficiently small such that the condition in Lemma~\ref{lem:tilde-r-n-xast} and Lemma~\ref{lem:prob_E_n_lager_1-delta-xast} are fulfilled. Moreover, we suppose $\gg_1$ small enough such that $\frac{4 C}{\mathbf{r}^2} \sum_{k=1}^n \gg_k^2\leq \frac{\delta}{2}$. 
    Then, we have 
    \begin{align*}
        \bbP(\hat E_n) \geq 1-\delta.
    \end{align*}
\end{lemma}
\begin{proof}
    By Lemma~\ref{lem:prob_hat-E_n_lager_1-delta}, we have $\bbP(E_n) \geq 1-\frac{\delta}{2}$. Moreover, by the additional step size assumption and Markov's inequality we deduce that 
    \begin{align*}
        \bbP(C_n) &= \bbP(\forall k \leq n\,:\, ||X_{k+1}-X_k||\leq \frac{\mathbf{r}}{2})\\
        &\geq 1- \sum_{k=1}^n \bbP(||X_{k+1}-X_k||> \frac{\mathbf{r}}{2})\\
        &= 1- \sum_{k=1}^n \bbP(||V_{k+1}(X_k)||> \frac{\mathbf{r}}{2 \gg_k})\\
        &\geq 1- \sum_{k=1}^n \E[||V_{k+1}(X_k)||^2] \frac{4 \gg_k^2}{\mathbf{r}^2 }\\
        &\geq 1- \frac{4 C}{\mathbf{r}^2 }\sum_{k=1}^n \gg_k^2 \\
        &\geq 1- \frac{\delta}{2 }.
    \end{align*}
    Together we obtain that $\bbP(\hat E_n) = 1- \bbP(\hat E_n^c) \geq 1- (\bbP(E_n^c) +\bbP(C_n^c)) \geq 1- \delta$.
\end{proof}
Finally, we are ready to prove the main result in the local setting for the set of local minima $\cX^\ast$.

\begin{proof}[Proof of Theorem~\ref{thm:conv-SGD-local-PL-xast}]
    (i): Recall the definitions of $\cU_1$ and $\cU$ above. Then it holds that
    \begin{equation*}
        \Omega_\cU = \bigcap_{n=1}^{\infty} \Omega_n.
    \end{equation*}
    Hence, using Lemma~\ref{lem:prob_hat-E_n_lager_1-delta} we obtain
    \begin{equation*}
        \bbP(\Omega_\cU) = \inf_n \bbP( \Omega_n) \geq \inf_n \bbP(\hat E_n) \geq 1-\delta.
    \end{equation*}

    (ii): We define $\tilde{D}_n := D_n \mathbf{1}_{\Omega_n} $ and prove that $\tilde D_n \in o(1/n^{1-\eta})$, then the claim follows since $\mathbf{1}_{\Omega_\cU} \leq \mathbf{1}_{\Omega_n}$ almost surely.

    From the proof of Lemma~\ref{lem:bar_D_iteration_x_ast} and Lemma~\ref{lem:tilde-r-n-xast} we have
    \begin{align*}
        \tilde{D}_{n+1} \le \tilde{D}_n  - \gg_n c \tilde{D}_n^{2\beta} + \gg_n \xi_{n+1} \mathbf{1}_{\Omega_n} + \frac{L \gg_n^2}{2} \lVert V_n \rVert^2 \mathbf{1}_{\Omega_n}.
    \end{align*}
    Hence, taking the conditional expectation gives
    \begin{align*}
        \E[\tilde{D}_{n+1}|\cF_n] &\le \tilde{D}_n  - \gg_n c \tilde{D}_n^{2\beta} + \gg_n \E[\xi_{n+1}  |\cF_n] \mathbf{1}_{\Omega_n}+ \frac{L \gg_n^2}{2} \E[\lVert V_{n+1}(X_n) \rVert^2 |\cF_n] \mathbf{1}_{\Omega_n}\\
        &\le \tilde{D}_n  - \gg_n c \tilde{D}_n^{2\beta} + \frac{LC}{2} \gg_n^2 ,
    \end{align*}
    where we have used that $D_n$ and $\mathbf{1}_{\Omega_n}$ are $\cF_n$-measurable and $E[\lVert V_{n+1}(X_n) \rVert^2 |\cF_n] \leq C$ from~\cref{eq:ABC-condition} with $A=B=0$.
    By our step size choice we can apply Lemma~\ref{lem:as_rate_extension_beta_allgemein} to obtain Claim (ii).

    (iii): In the following, we again separate between the two cases of $\beta$.
    
    \underline{$\beta = \frac{1}{2}$:} We have from Lemma~\ref{lem:bar_D_iteration_x_ast} \cref{eq:lemE1-eq3-xast} and Lemma~\ref{lem:tilde-r-n-xast} that
    \begin{align*}
        \tilde{D}_{n+1} &\le (1- \gg_n c^2) \tilde{D}_n  + \gg_n \xi_{n+1} \mathbf{1}_{\Omega_n} + \frac{L \gg_n^2}{2} \lVert V_{n+1}(X_n) \rVert^2 \mathbf{1}_{\Omega_n}.
    \end{align*}
    Taking expectations and multiplying by $(n+1)^{1-\eta}$ leads to
    \begin{align*}
        &\E[\tilde{D}_{n+1}(n+1)^{1-\eta} ] \\
        &\le (n+1)^{1-\eta} (1- \gg_n c^2) \E[\tilde{D}_n]  + (n+1)^{1-\eta} \frac{L C \gg_n^2}{2}\\
        & \le \left(n^{1-\eta} +(1-\eta) n^{-\eta} \right)(1- \gg_n c^2) \E[\tilde{D}_n]  + (n+1)^{1-\eta} \frac{L C \gg_n^2}{2}\\
        &= \left(n^{1-\eta} +(1-\eta) n^{-\eta} - n^{1-\eta} \gg_n c^2 - (1-\eta) n^{-\eta}\gg_n c_{x^\ast}^2\right) \E[\tilde{D}_n]  + (n+1)^{1-\eta}\gg_n^2 \frac{L C}{2}\\
        &= \left(1 +\frac{1-\eta}{n} - \gg_n c^2 - \frac{(1-\eta) \gg_n c^2}{n}\right) n^{1-\eta} \E[\tilde{D}_n]  + (n+1)^{1-\eta}\gg_n^2 \frac{L C}{2},
    \end{align*}
    where we used \cref{eq:V_n^2-abschaetung} in the first inequality.
    By our choice of $\gg_n$ there exists $\tilde c>0$ and $N>0$ such that $ \gg_n c^2 -\frac{1-\eta}{n} + \frac{(1-\eta) \gg_n c^2}{n} \geq \tilde c \gg_n$ for all $n\geq N$. Thus, for all $n \geq N$ 
    \begin{align*}
        w_{n+1} &\le \left(1  - \tilde c \gg_n \right) w_n  + (n+1)^{1-\eta}\gg_n^2 \frac{L C}{2},
    \end{align*}
    where $w_n = \E[n^{1-\eta} \tilde D_n]$. Define $a_n = \tilde c \gg_n$ and $b_n = (n+1)^{1-\eta}\gg_n^2 \frac{L C}{2}$. Since  $\gg_n = \Theta(\frac{1}{n^\gt})$, we have $\sum_n a_n = \tilde c\sum_n \gg_n = \infty$ and 
    \begin{align*}
        \sum_n b_n = \frac{L C}{2} \sum_n (n+1)^{1-\eta}\gg_n^2 < \infty,\, 
    \end{align*}
    by \cref{eq:cond-2-lem21} in Lemma~\ref{lem:as_rate_extension_beta_allgemein}
    Hence, we apply Lemma~\ref{lem:help-an-bn-cn} to prove that $\lim_{n\to\infty} w_n=0$. By the definition of $w_n$ we have verified that $\E[(f(X_n) - l ) \mathbf{1}_{\Omega_{\cU}}] \leq \E[\tilde D_n]\in o(\frac{1}{n^{1-\eta}})$
    
    \underline{$\beta \in (\frac{1}{2},1]$:} From Lemma~\ref{lem:bar_D_iteration_x_ast} \cref{eq:lemE1-eq1-xast} and Lemma~\ref{lem:tilde-r-n-xast} we have
    \begin{align*}
        \tilde{D}_{n+1} &\le (1- \gg_n^q c^2) \tilde{D}_n + \tilde c \gg_n^\frac{2\beta q-1}{2\beta-1} + \gg_n \xi_{n+1} \mathbf{1}_{\Omega_n} + \frac{L \gg_n^2}{2} \lVert V_{n+1}(X_n) \rVert^2 \mathbf{1}_{\Omega_n},
    \end{align*}
    for $\tilde c = (2\beta)^{-\frac{1}{2\beta-1}} (1-\frac{1}{2\beta}) c^2$.
    Next we multiply with $(n+1)^{1-\eta}$ and use \cref{eq:elementary} to obtain
    \begin{align*}
        &\E[\tilde{D}_{n+1} (n+1)^{1-\eta}] \\
        &\le (n+1)^{1-\eta} (1- \gg_n^q c^2) \E[\tilde{D}_n] + (n+1)^{1-\eta} \tilde c \gg_n^\frac{2\beta q-1}{2\beta-1}  + (n+1)^{1-\eta} \frac{LC}{2} \gg_n^2\\
        &\le \left(n^{1-\eta} +(1-\eta) n^{-\eta} \right) \big(1-\gamma_n^q c^2\big)  \E[\tilde{D}_n] + c_1 (n+1)^{1-\eta} (\gg_n^\frac{2\beta q-1}{2\beta-1} + \gg_n^2)\\
        &= \left(n^{1-\eta} +(1-\eta) n^{-\eta} - \gamma_n^q c^2 n^{1-\eta} - (1-\eta)\gamma_n^q c^2 n^{-\eta}\right)  \E[\tilde{D}_n] \\
        &\quad+ c_1 (n+1)^{1-\eta} (\gg_n^\frac{2\beta q-1}{2\beta-1} + \gg_n^2)\\
        &= \E[\tilde{D}_n n^{1-\eta}] \left(1 +\frac{1-\eta}{n} - \gamma_n^q c^2 - \frac{(1-\eta)\gamma_n^q c^2 }{n}\right)+ c_1 (n+1)^{1-\eta} (\gg_n^\frac{2\beta q-1}{2\beta-1} + \gg_n^2),
    \end{align*}
    for some $c_1>0$. 
    By our choice of $\gg_n$ and as $q\geq 1$, there exists a $c_2>0$ and $N>0$ such that $\gamma_n^q c^2 - \frac{1-\eta}{n}+ \frac{(1-\eta)\gamma_n^q c^2 }{n} \geq c_2 \gg_n^q$ for all $n\geq N$. Thus, for $n\geq N$
    \begin{align*}
        &\E[\tilde{D}_{n+1} (n+1)^{1-\eta}] \le \E[\tilde{D}_n n^{1-\eta}] \left(1 - c_2 \gg_n^q \right)+ c_1 (n+1)^{1-\eta} (\gg_n^\frac{2\beta q-1}{2\beta-1} + \gg_n^2).\\
    \end{align*}
    Define $w_n = \E[\tilde{D}_n n^{1-\eta}]$, 
     $a_n =c_2 \gg_n^q $ and $b_n = c_1 (n+1)^{1-\eta} (\gg_n^\frac{2\beta q-1}{2\beta-1} + \gg_n^2)$.
    We will again apply Lemma~\ref{lem:help-an-bn-cn}. By the step size choice $\gg_n = \Theta(\frac{1}{n^\gt})$ we have $\sum_n a_n = c_2 \sum_n \gg_n^q = \infty$, because $q\leq \frac{1}{\gt}$. Further,
    \begin{align*}
        \sum_n b_n = c_1 \sum_n (n+1)^{1-\eta} (\gg_n^\frac{2\beta q-1}{2\beta-1} + \gg_n^2) < \infty, 
    \end{align*}
    because we choose the auxiliary parameter $q$ as in the proof of Lemma~\ref{lem:as_rate_extension_beta_allgemein} where we showed in \cref{eq:cond-2-lem21} and \cref{eq:cond-3-lem21} that 
    \begin{align*}
        \sum_{n=N}^\infty n^{1-\eta-2\gt} <\infty\quad \textrm{and} \quad \sum_{n=N}^\infty n^{1-\eta-\frac{\gt (2\beta q-1)}{2\beta-1}}<\infty
    \end{align*}
    All together we deduce that $w_n$ vanishes at infinity. Again, by the definition of $w_n$ we have that $\E[(f(X_n) - l ) \mathbf{1}_{\Omega_{\cU}}] \leq \E[\tilde D_n] \in o(\frac{1}{n^{1-\eta}})$    
\end{proof}

\subsection{Proof of Theorem~\ref{thm:conv-SGD-local-PL-fast}}\label{app:subsec-thm5.2}
Suppose throughout this section that the assumptions in \cref{thm:conv-SGD-local-PL-fast} are satisfied.

The proof will be similar to the previous section. Instead of assuring that $(X_n)$ remains close to the set where we could guarantee the unified gradient domination property, it is now sufficient that $f(X_n)$ remains close to $f^\ast$ by the different definition of gradient domination definition in $f^\ast$. This will simplify the proof. Moreover, we may again assume w.l.o.g.~the uniform second moment bounds, i.e.~$A=B=0$, instead of the more general \cref{eq:ABC-condition} condition by the same argument as above but on the level sets.

Recall the notation
\[ \cB_r^\ast = \{x\in\R^d\,:\, f(x)-f^\ast \leq r\}. \]
and let $r>0$ be the radius of the gradient domination property in $f^\ast$, then there exists $\epsilon >0$, such that $2 \epsilon +\sqrt{\epsilon}<r$, i.e
\begin{equation}\label{eq:B_epsilon_ast}
    \cU := \cB_{2 \epsilon +\sqrt{\epsilon}}^\ast \subset \cB_r^\ast.
\end{equation}
Moreover, we define the set 
\begin{align}
    \cU_1 := \cB_{\frac \epsilon 2}^\ast
\end{align}
and the measurable subsets 
\[ \Omega_n = \{ X_k \in \cU, \textrm{ for all } k \leq n\} \]
in $(\Omega,\cF,\bbP)$.

The proof of \cref{thm:conv-SGD-local-PL-fast} is again based on a series of auxiliary lemmas. The goal of these is to prove that with high probability we do not leave the gradient dominated region, i.e. Claim (i) in \cref{thm:conv-SGD-local-PL-fast}. 

In the following, we fix the notation $D_n := f(X_n)-f^\ast$ and $\tilde D_n := D_n \mathbf{1}_{\Omega_n}$, $\xi_{n+1}:=-\langle \nabla f(X_n),Z(X_n,\zeta_{n+1})\rangle$ and obtain the parallel result to Lemma~\ref{lem:bar_D_iteration_x_ast}. 

\begin{lemma}\label{lem:bar_D_iteration_f_ast}
    If $\beta = \frac 12$, it holds that
    \begin{equation}\label{eq:lemE1-eq3}
        \begin{split}
        {D}_{n+1} \mathbf{1}_{\Omega_n} &\le (1-\gg_n c^2) D_n \mathbf{1}_{\Omega_n}  + \gg_n \xi_{n+1}\mathbf{1}_{\Omega_n} + \frac{L \gg_n^2}{2} \mathbf{1}_{\Omega_n} \lVert V_{n+1}(X_n) \rVert^2, \\
        &\le {D}_1 \prod_{k=1}^{n}(1- \gg_k c^2)\mathbf{1}_{\Omega_n} + \sum_{k=1}^{n} \left(\prod_{j=k}^{n}(1- \gg_j c^2)\right) \gg_k \xi_{k+1} \mathbf{1}_{\Omega_n}\\ &\quad+ \frac{L}{2}\sum_{k=1}^{n} \gg_k^2 \lVert V_{k+1}(X_k) \rVert^2\mathbf{1}_{\Omega_n}.
        \end{split}
    \end{equation}
    If $\beta \in (\frac{1}{2},1]$, for any $1 \leq q\leq 2$, it holds that
    \begin{equation}\label{eq:lemE1-eq1}
        \begin{split}
        {D}_{n+1}\mathbf{1}_{\Omega_n} &\le (1- \gg_n^q c^2) {D}_n \mathbf{1}_{\Omega_n} +  (2\beta)^{-\frac{1}{2\beta-1}} (1-\frac{1}{2\beta}) {c^2}\gg_n^\frac{2\beta q-1}{2\beta-1}+ \gg_n \xi_n \mathbf{1}_{\Omega_n} + \frac{L \gg_n^2}{2} \lVert V_{n+1}(X_n) \rVert^2 \\
        &\le {D}_1 \prod_{k=1}^{n}(1- \gg_k^q c^2)\mathbf{1}_{\Omega_n} +  \tilde c \sum_{k=1}^{n} \gg_k^\frac{2\beta q-1}{2\beta-1} + \sum_{k=1}^{n} \left(\prod_{j=k}^{n}(1- \gg_j^q c^2)\right) \gg_k \xi_{k+1} \mathbf{1}_{\Omega_n}\\&\quad + \frac{L}{2}\sum_{k=1}^{n}  \gg_k^2 \lVert V_{k+1}(X_k) \rVert^2\mathbf{1}_{\Omega_n},
        \end{split}
    \end{equation}
    for $\tilde c = (2\beta)^{-\frac{1}{2\beta-1}} (1-\frac{1}{2\beta}) c^2$.
\end{lemma}

\begin{proof}
The proof follows line for line as in Lemma~\ref{lem:bar_D_iteration_x_ast} by replacing $l$ with $f^\ast$ and taking the different definition of $\tilde D_n$ and $\Omega_n$ into account. 
\end{proof}

We continue as in the previous section:\\
For $\beta >\frac 12$ we know from the proof of Lemma~\ref{lem:as_rate_extension_beta_allgemein} that we can choose the auxiliary parameter $q$ from the previous lemma in such a way, that $\sum_{n=1}^{\infty} n^{1-\eta} \gg_n^\frac{2\beta q-1}{2\beta-1}$ is convergent for all $\eta \in (\max\{2-2\gt,\frac{\gt+2\beta-2}{2\beta-1}\},1)$ (Condition (iii) to apply Lemma~\ref{cor:RS_corollary}). As
 $\eta<1$, it follows that $\sum_{n=1}^{\infty} \gg_n^\frac{2\beta q-1}{2\beta-1} <\infty$ holds true for all these choices of $q$. Now define 
\begin{align*} 
&M_n = \sum_{k=1}^{n} \left(\prod_{j=k}^{n}(1- \gg_j c^2)\right)\gg_k \xi_{k+1}  \mathbf{1}_{\Omega_k}, \quad M_n^{(q)} = \sum_{k=1}^{n} \left(\prod_{j=k}^{n}(1- \gg_j^q c^2)\right)\gg_k \xi_{k+1}  \mathbf{1}_{\Omega_k} \quad \\
& \textrm{ and } \quad S_n = \frac{L}{2}\sum_{k=1}^{n}  \gg_k^2 \lVert V_{k+1}(X_k) \rVert^2\mathbf{1}_{\Omega_k}.
\end{align*}
Then, $(M_n)_{n\in\N}$ and $(M_n^{(q)})$ are $(\cF_{n+1})$-martingales with zero mean and $(S_n)_{n\in\N}$ is a $(\cF_{n+1})$-sub-martingale by Assumption~\ref{ass:ABC}. Note that by the choice of $\gg_n$ we have $\sum_n \gg_n^2 <\infty$ and hence $\E[S_n] < \infty$ for all $n \in\N$. 
Next, define $R_n = M_n^2 + S_n$ and $R_n = (M_n^{(q)})^2+S_n$ respectively (with some abuse of notation) for every $n\in\N$. Moreover, let 
\begin{equation*}
    E_n = \{ R_k < \epsilon \textrm{ for all } k\leq n \}.
\end{equation*}
which is a $\cF_{n+1}$-measurable event on $(\Omega, \cF,\bbP)$.
We define $R_0 = 0$ such that $E_0 = \Omega$.

With these definitions, we can directly prove a parallel result to Lemma~\ref{lem:tilde-r-n-xast} without the auxiliary result in Lemma~\ref{lem:aux-lemma-x_n-close-to-xast}.
\begin{lemma}\label{lem:tilde-r-n}
    For $\beta \in (\frac{1}{2},1]$ let $\gg_n \leq \gg_1$ be sufficiently small such that $\sum_{n=1}^{\infty} \gg_n^\frac{2\beta q-1}{2\beta-1} < \frac{\epsilon}{2 \tilde c}$, and for $\beta=\frac{1}{2}$ let $\gg_1>0$ be arbitrary. Furthermore, assume that the initial $X_1 \in \cU_1 =\{x: f(x) -f(x^\ast) \leq \frac \epsilon 2\}$ almost surely.  Then, 
    \begin{enumerate}
        \item[a)] $E_{n+1} \subset E_n$ and $\Omega_{n+1} \subset \Omega_n$
        \item[b)] $E_n \subset \Omega_{n+1}$
        \item[c)] Define the events $\tilde{E}_n = E_{n-1} \setminus E_n = E_{n-1} \cup \{R_n\geq\epsilon\}$. Then, for $\tilde{R}_n = R_n \mathbf{1}_{E_{n-1}}$, there exists a $\tilde C>0$ such that
        \begin{equation*}
            \E[\tilde{R}_{n}] \le \E[\tilde{R}_{n-1}] + \gamma_n^2 [G^2 C^2 + G^2 + C] - \epsilon \bbP(\tilde{E}_{n-1}).
        \end{equation*}
    \end{enumerate}
\end{lemma}

\begin{proof}
   a) Follows by definition. 
   
   b) Note that $E_0 = \Omega = \Omega_1$ because $X_1 \in \cU_1 = \Omega_1$ almost surely.
    We prove the claim by induction. Let $\omega \in E_n$. Since $E_n \subset E_{n-1} \subset \Omega_n$ by induction assumption, we have $\omega \in \Omega_n$ and thus $\omega\in \Omega_k$ for all $k \le n$. It remains to show that $X_{n+1}(\omega) \in \cU$ to prove that $\omega \in \Omega_{n+1}$. We separate both cases for $\beta$:
    
    \underline{$\beta =\frac{1}{2}$:}
    The inequality \cref{eq:lemE1-eq3} and the induction hypothesis yield 
    \begin{align*}
        D_{n+1}(\omega) 
        &\le D_1(\omega) \prod_{k=1}^{n}(1- \gg_k c) + \sum_{k=1}^{n} \left(\prod_{j=k}^{n}(1- \gg_j c^2)\right) \gg_k \xi_{k+1}(\omega) + \frac{L}{2}\sum_{k=1}^{n}  \gg_k^2 \lVert V_{k+1}(X_k(\omega)) \rVert^2 \\
        &= D_1(\omega) \prod_{k=1}^{n}(1- \gg_k c) + \sum_{k=1}^{n} \left(\prod_{j=k}^{n}(1- \gg_j c^2)\right) \gg_k \xi_{k+1}(\omega) \mathbf{1}_{\Omega_k}(\omega)\\ &\quad+ \frac{L}{2}\sum_{k=1}^{n}  \gg_k^2 \lVert V_{k+1}(X_k(\omega)) \rVert^2  \mathbf{1}_{\Omega_k}(\omega)\\
        &\le \frac{\epsilon}{2} +\sqrt{R_n(\omega)} + R_n(\omega) \\
        &\le 2 \epsilon + \sqrt{\epsilon}.
    \end{align*}
    Hence, $X_{n+1}(\omega)\in \cU$ by definition of $\cU$.
    
    \underline{$\beta \in (\frac{1}{2},1]$:}  Similarly, we obtain from \cref{eq:lemE1-eq1}
    \begin{align*}
        D_{n+1}(\omega) &\le D_1(\omega) \prod_{k=1}^{n}(1- \gg_k^q c) +\tilde c \sum_{k=1}^{n} \gg_k^\frac{2\beta q-1}{2\beta-1}+ \sum_{k=1}^{n} \left(\prod_{j=k}^{n}(1- \gg_j^q c^2)\right)\gg_k \xi_{k+1}(\omega)\\ 
        &\quad+ \frac{L}{2}\sum_{k=0}^{n}  \gg_k^2 \lVert V_{k+1}(X_k(\omega)) \rVert^2 \\
        &= D_1(\omega) \prod_{k=1}^{n}(1- \gg_k^q c) +\tilde c \sum_{k=1}^{n} \gg_k^\frac{2\beta q-1}{2\beta-1}+ \sum_{k=1}^{n} \left(\prod_{j=k}^{n}(1- \gg_j^q c^2)\right)\gg_k \xi_{k+1}(\omega) \mathbf{1}_{\Omega_k}(\omega)\\ &\quad+ \frac{L}{2}\sum_{k=0}^{n}  \gg_k^2 \lVert V_{k+1}(X_k(\omega)) \rVert^2  \mathbf{1}_{\Omega_k}(\omega)\\
        &\le \frac{\epsilon}{2}+\frac{\epsilon}{2}+ \sqrt{R_n(\omega)} + R_n(\omega) \\
        &\le 2 \epsilon + \sqrt{\epsilon},
    \end{align*}
    where we used that $\prod_{k=1}^{n}(1- \gg_k^q c^\ast)<1$.
    Hence, it holds again that $X_{n+1}(\omega) \in \cU$.
    
    This prove that $\omega\in\Omega_{n+1}$ and closes the induction. 
    
    c) Follows line by line as in Lemma~\ref{lem:tilde-r-n-xast} part c).
\end{proof}

\begin{lemma}\label{lem:prob_E_n_lager_1-delta}
    Let $\delta >0$ be a tolerance level and $\gg_n \leq \gg_1$ be sufficiently small such that $\sum_{n=1}^{\infty} \gg_n^2 < \frac{\delta \epsilon}{2(G^2 C^2 + G^2 + C)}$ and the condition in Lemma~\ref{lem:tilde-r-n} is fulfilled. 
    Then, we have 
    \begin{align*}
        \bbP(E_n) \geq 1-\delta.
    \end{align*}
\end{lemma}
\begin{proof}
    Line by line as in Lemma~\ref{lem:prob_E_n_lager_1-delta-xast}.
\end{proof}

Finally, we are ready to prove the main result in the local setting for $f^\ast$.

\begin{proof}[Proof of Theorem~\ref{thm:conv-SGD-local-PL-fast}]
    (i): Recall the definition of $\cU_1$ and $\cU$ above. Then it holds that
    \begin{equation*}
        \Omega_\cU = \bigcap_{n=1}^{\infty} \Omega_n.
    \end{equation*}
    Hence, using Lemma~\ref{lem:prob_E_n_lager_1-delta} we obtain
    \begin{equation*}
        \bbP(\Omega_\cU) = \inf_n \bbP( \Omega_n) \geq \inf_n \bbP(E_n) \geq 1-\delta.
    \end{equation*}

    The proof of Claim (ii) and (iii) follows line by line as in the proof of \cref{thm:conv-SGD-local-PL-xast} by replacing $l$ with $f^\ast$ and taking the different definitions of $D_n$, $\tilde D_n$ and $\Omega_n$ into account.  
\end{proof}

\section{Proofs of Section~\ref{sec:application-RL}}\label{app:proof-localPL_RL}

In the following section, we provide the proof of \Cref{sec:application-RL}.
\begin{proof}[Proof of~\Cref{lem:localPL}]
    We consider the cases $\lambda =0$ and $\lambda >0 $ separately.\\
    \textbf{Case $\lambda=0$:} Define the optimal reward gap in every state $s\in\cS$ by 
    \[ \Delta^\ast(s) = Q^{\ast}(s,a^\ast(s)) - \max_{a\neq a^\ast(s)} Q^{\ast}(s,a) >0, \]
    where $a^\ast(s)$ denotes the best possible action in state $s$ and $Q^\ast: \cS \times \cA \to \R$ denotes the optimal Q-function defined by $Q^\ast(s,a) = \E_\mu^{\pi^\ast}[\sum_{t=0}^\infty \rho^{-t} r(S_t,A_t) |A_0=a]$. W.l.o.g. we assume that $a^\ast(s)$ is unique. Similarly let $Q^{\pi_w}(s,a)= \E_\mu^{\pi_w}[\sum_{t=0}^\infty \rho^{-t} r(S_t,A_t) |A_0=a]$ be the Q-function for policy $\pi_w$.
    
    For any $0<\alpha<1$ choose $r= \min_{s\in\cS} \mu(s) \min_{s\in\cS} \Delta^\ast(s) (1-\alpha)$ and assume that $w \in \cB_r^\ast$, i.e.  $V^\ast(\mu)-V^{\pi_w} (\mu)\le r $. Then, we have for every $s\in\cS$ that
    \[V^\ast(\delta_s)-V^{\pi_w} (\delta_s)\le \frac{r}{\min_{s\in\cS} \mu(s)}.\]
    It follows for every $s\in\cS$ that
    \begin{align*}
        \frac{r}{\min_{s\in\cS} \mu(s)}
        &\geq V^\ast(\delta_s)-V^{\pi_w} (\delta_s) \\
        &= Q^{\ast}(s,a^\ast(s)) - \sum_{a\in\cA_s} \pi_w(a|s) Q^{\pi_w}(s,a)  \\
        &\geq  \sum_{a\in\cA_s} \pi_w(a|s) ( Q^{\ast}(s,a^\ast(s)) - Q^{\ast}(s,a) ) \\
        &= \sum_{a\neq a^\ast(s)} \pi_w(a|s)( Q^{\ast}(s,a^\ast(s)) - Q^{\ast}(s,a) ) \\
        &\geq (1-\pi_w(a^\ast(s)|s) \min_s \Delta^\ast(s). 
    \end{align*}
    Rearranging results in 
    \[ \pi_w(a^\ast(s)|s) \geq 1- \frac{r}{\min_{s\in\cS} \mu(s) \min_{s\in\cS} \Delta^\ast(s)} = \alpha.\]
    Hence, for all $w \in \cB_r^\ast$ we can bound $c(w)$ by
    \begin{align*}
        c(w)\geq \frac{\alpha}{\sqrt{|\cS|}(1-\rho) }\Big\lVert\frac{d_\mu^{\pi^\ast}}{\mu} \Big\rVert_\infty^{-1}>0.
    \end{align*}
    Thus, setting $c = \frac{\alpha}{\sqrt{|\cS|}(1-\rho) }\Big\lVert\frac{d_\mu^{\pi^\ast}}{\mu} \Big\rVert_\infty^{-1}$ proves the claim. 
    
    \textbf{Case $\lambda>0$:} For any $\alpha \in (0,1)$ choose $r = \alpha^2 \exp\left(\frac{-1}{(1-\rho)\lambda}\right)^2 \frac{\lambda \min_s \mu(s)}{2 \mathrm{ln}2} $ and assume that $w\in \cB_{r,\lambda}^\ast$. By \citet[Lem. 12]{ding2023local} we have 
    \begin{align*}
        | \pi_w(a|s) - \pi^\ast(a|s)|
        & \leq \sqrt{ \frac{2 (V_\lambda^\ast(\mu)- V_\lambda^{\pi_w}(\mu)) \mathrm{ln} 2}{\lambda \min_s \mu(s)}}\\
        &\leq \sqrt{ \frac{2 r \mathrm{ln} 2}{\lambda \min_s \mu(s)}}\\
        &= \alpha \exp\left(\frac{-1}{(1-\rho)\lambda}\right) \\
        &\leq \alpha \min_{s,a} \pi^\ast(a|s).
    \end{align*}
    where the last inequality is due to \citet[Thm. 1]{nachum2017bridging}. It follows directly that 
    \begin{align*}
        \min_{s,a} \pi_w(s,a) \geq (1-\alpha) \min_{s,a}\pi^\ast(s,a) >0.
    \end{align*}
    Hence, we can bound $c(w)$ uniformly for all $w \in \cB_{r,\lambda}^\ast$ by
    \begin{align}
        c(w) \geq  \frac{2\lambda}{|\cS|(1-\rho)} \min_s \mu(s) (1-\alpha)^2 \min_{s,a} \pi^\ast(a|s)^2 \Big\lVert\frac{d_\mu^{\pi^\ast}}{\mu} \Big\rVert_\infty^{-1}.
    \end{align}
    Thus, setting $c = \frac{2\lambda}{|\cS|(1-\rho)} \min_s \mu(s) (1-\alpha)^2 \min_{s,a} \pi^\ast(a|s)^2 \Big\lVert\frac{d_\mu^{\pi^\ast}}{\mu} \Big\rVert_\infty^{-1}$ proves the claim.
\end{proof}

\begin{remark}\label{rem:c-r-depend-on-alpha}
    It is noteworthy, that we have multiple choices of $r$ and $c$ depending on $\alpha \in (0,1)$.
\end{remark}

\end{document}